\documentclass{article}
\usepackage[final]{nips_2018}

\usepackage[utf8]{inputenc} %
\usepackage[T1]{fontenc}    %
\usepackage{hyperref}       %
\usepackage{url}            %
\usepackage{booktabs}       %
\usepackage{amsfonts}       %
\usepackage{nicefrac}       %
\usepackage{microtype}      %
\usepackage{color}
\usepackage{amsmath}
\usepackage{amsthm}
\usepackage{graphicx}
\usepackage{subcaption}
\usepackage{tikz}
\usepackage{caption}
\usepackage{multirow}
\usepackage{wrapfig}
\usepackage{algorithm, algpseudocode}%
\usepackage{amsfonts, amsmath}
\usepackage{enumitem}
\usepackage{subcaption}
\usepackage{soul}
\usepackage{bm}
\usepackage{wrapfig}
\usepackage{listings}
\usepackage[title]{appendix}
\usepackage[bottom]{footmisc}
\usepackage[makeroom]{cancel}

\setcitestyle{authoryear,open={(},close={)}}

\definecolor{mydarkblue}{rgb}{0,0.08,0.45}
\definecolor{myfavblue}{rgb}{0.1176, 0.392, 1.0}
\hypersetup{
    colorlinks=true,
    linkcolor=mydarkblue,
    citecolor=mydarkblue,
    filecolor=mydarkblue,
    urlcolor=mydarkblue}

\floatname{algorithm}{Algorithm}

\algtext*{EndFunction}

\newtheorem{theorem}{Theorem}
\newtheorem*{theorem*}{Theorem}
\newcommand{\tr}{\textnormal{tr}}
\newcommand\E{\mathbb{E}}
\newcommand\R{\mathbb{R}}

\renewcommand{\epsilon}{\varepsilon}

\newcommand*{\tran}{^{\mkern-1.5mu\mathsf{T}}}

\newcommand{\vx}{\mathbf{x}}

\newcommand\bigO{\mathcal{O}}
\newcommand{\va}{\mathbf{a}}
\newcommand{\vh}{\mathbf{h}}
\newcommand{\vz}{\mathbf{z}}

\newcommand{\vzero}{\bf{0}}

\DeclareMathOperator{\KLop}{KL}
\newcommand{\KL}[2]{\KLop \left(#1 \middle \| #2 \right)}

\newcommand{\latent}{\vz}
\newcommand{\hidden}{\vh}
\newcommand{\obs}{\vx}
\newcommand{\sol}{\vz}  %

\newcommand{\solvefunc}{\textnormal{ODESolve}}
\newcommand{\tstart}{{t_\textnormal{0}}}
\newcommand{\tend}{{t_\textnormal{1}}}
\newcommand{\lograte}{\lambda}%
\newcommand{\method}{Latent ODE}
\newcommand{\cnfx}{\sol}
\newcommand{\adj}{\va}

\title{Neural Ordinary Differential Equations}

\author{
  Ricky T. Q. Chen*, Yulia Rubanova*, Jesse Bettencourt*, David Duvenaud \\
  University of Toronto, Vector Institute\\
  \texttt{ \{rtqichen, rubanova, jessebett, duvenaud\}@cs.toronto.edu} \\
}

\begin{document}
\maketitle

\definecolor{jessecomment}{HTML}{85144b}
\newcommand{\JB}[1]{{\bf \color{jessecomment}JB: [#1]}}

\begin{abstract}
We introduce a new family of deep neural network models.
Instead of specifying a discrete sequence of hidden layers, we parameterize the derivative of the hidden state using a neural network.
The output of the network is computed using a black-box differential equation solver.
These continuous-depth models have constant memory cost, adapt their evaluation strategy to each input, and can explicitly trade numerical precision for speed.
We demonstrate these properties in continuous-depth residual networks and continuous-time latent variable models.
We also construct continuous normalizing flows, a generative model that can train by maximum likelihood, without partitioning or ordering the data dimensions.
For training, we show how to scalably backpropagate through any ODE solver, without access to its internal operations.
This allows  end-to-end training of ODEs within larger models.%
\end{abstract}

\begin{wrapfigure}[18]{r}{0.45\textwidth}
\centering
\vspace{-1cm}%
\hspace{-3mm}%
\begin{tabular}{cc}
\setlength{\tabcolsep}{0pt}
\hspace{2mm} Residual Network & \hspace{2mm} ODE Network \\
\includegraphics[width=0.2\textwidth, clip, trim=4mm 4mm 4mm 4mm]{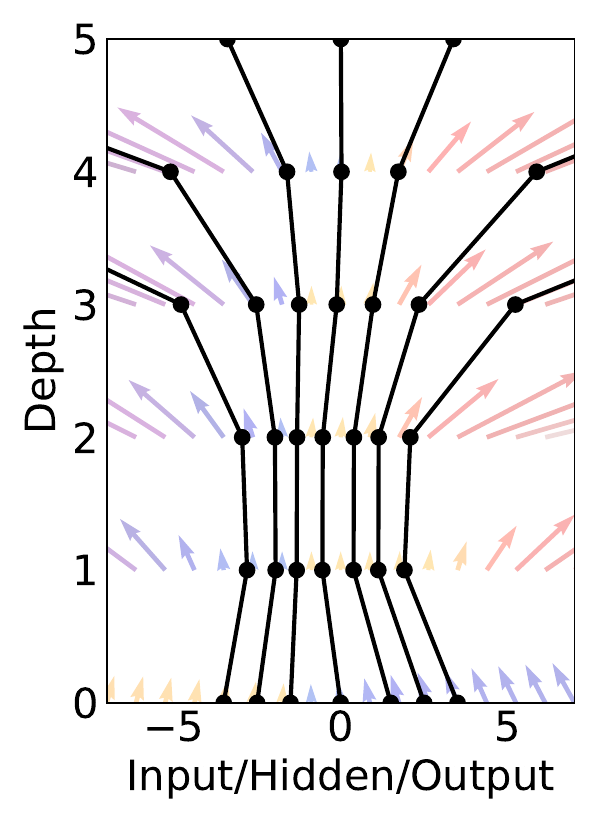} &
\includegraphics[width=0.2\textwidth, clip, trim=4mm 4mm 4mm 4mm]{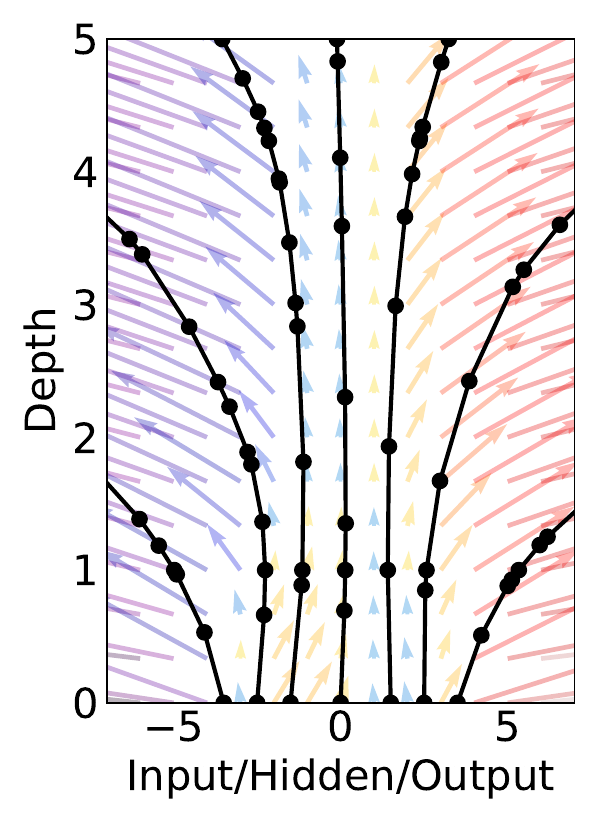}
\end{tabular}
\caption{\emph{Left:} A Residual network defines a discrete sequence of finite transformations.
\emph{Right:}
An ODE network defines a vector field, which continuously transforms the state.
\emph{Both:}
Circles represent evaluation locations.}
\label{fig:fig1}
\end{wrapfigure}

\section{Introduction}
Models such as residual networks, recurrent neural network decoders, and normalizing flows build complicated transformations by composing a sequence of transformations to a hidden state:
\begin{align}\label{eq:res}
\hidden_{t+1} = \hidden_t + f(\hidden_t, \theta_t)%
\end{align}
where $t \in \{0 \dots T\}$ and $\hidden_t\in \R^D$.
These iterative updates can be seen as an Euler discretization of a continuous transformation~\citep{lu2017beyond,haber2017stable,ruthotto2018deep}.

What happens as we add more layers and take smaller steps?
In the limit, we parameterize the continuous dynamics of hidden units using an ordinary differential equation (ODE) specified by a neural network:
\begin{align}\label{eq:ivp}
\frac{d\hidden(t)}{dt} = f(\hidden(t), t, \theta) %
\end{align}
Starting from the input layer $\hidden(0)$, we can define the output layer $\hidden(T)$ to be the solution to this ODE initial value problem at some time $T$.
This value can be computed by a black-box differential equation solver, which evaluates the hidden unit dynamics $f$ wherever necessary to determine the solution with the desired accuracy.
Figure~\ref{fig:fig1} contrasts these two approaches.

Defining and evaluating models using ODE solvers has several benefits:

\paragraph{Memory efficiency}
In Section~\ref{sec:rmd}, we show how to compute gradients of a scalar-valued loss with respect to all inputs of any ODE solver, \emph{without backpropagating through the operations of the solver}.
Not storing any intermediate quantities of the forward pass allows us to train our models with constant memory cost as a function of depth, a major bottleneck of training deep models.

\paragraph{Adaptive computation}
Euler's method is perhaps the simplest method for solving ODEs.
There have since been more than 120 years of development of efficient and accurate ODE solvers~\citep{Runge, Kutta, hairer87:_solvin_ordin_differ_equat_i}.
Modern ODE solvers provide guarantees about the growth of approximation error, monitor the level of error, and adapt their evaluation strategy on the fly to achieve the requested level of accuracy.
This allows the cost of evaluating a model to scale with problem complexity.
After training, accuracy can be reduced for real-time or low-power applications.


\paragraph{Scalable and invertible normalizing flows}
An unexpected side-benefit of continuous transformations is that the change of variables formula becomes easier to compute.
In Section~\ref{sec:cnf}, we derive this result and use it to construct a new class of invertible density models that avoids the single-unit bottleneck of normalizing flows, and can be trained directly by maximum likelihood.

\paragraph{Continuous time-series models}
Unlike recurrent neural networks, which require discretizing observation and emission intervals, continuously-defined dynamics can naturally incorporate data which arrives at arbitrary times.
In Section~\ref{sec:lode}, we construct and demonstrate such a model.

\section{Reverse-mode automatic differentiation of ODE solutions}
\label{sec:rmd}
The main technical difficulty in training continuous-depth networks is performing reverse-mode differentiation (also known as backpropagation) through the ODE solver.
Differentiating through the operations of the forward pass is straightforward, but incurs a high memory cost and introduces additional numerical error.

We treat the ODE solver as a black box, and compute gradients using the \emph{adjoint sensitivity method}~\citep{pontryagin1962mathematical}.
This approach computes gradients by solving a second, augmented ODE backwards in time, and is applicable to all ODE solvers.
This approach scales linearly with problem size, has low memory cost, and explicitly controls numerical error.

Consider optimizing a scalar-valued loss function $L()$, whose input is the result of an ODE solver:
\begin{align}
L(\sol(\tend)) = L \left(\sol(\tstart) + \int_{\tstart}^{\tend} f(\sol(t), t, \theta) dt \right) = L \left( \solvefunc( \sol(\tstart), f, \tstart, \tend, \theta) \right)
\end{align}
\begin{wrapfigure}[23]{l}{0.5\textwidth}%
	\vspace{-3mm}
	\centering
	\includegraphics[width=0.5\textwidth]{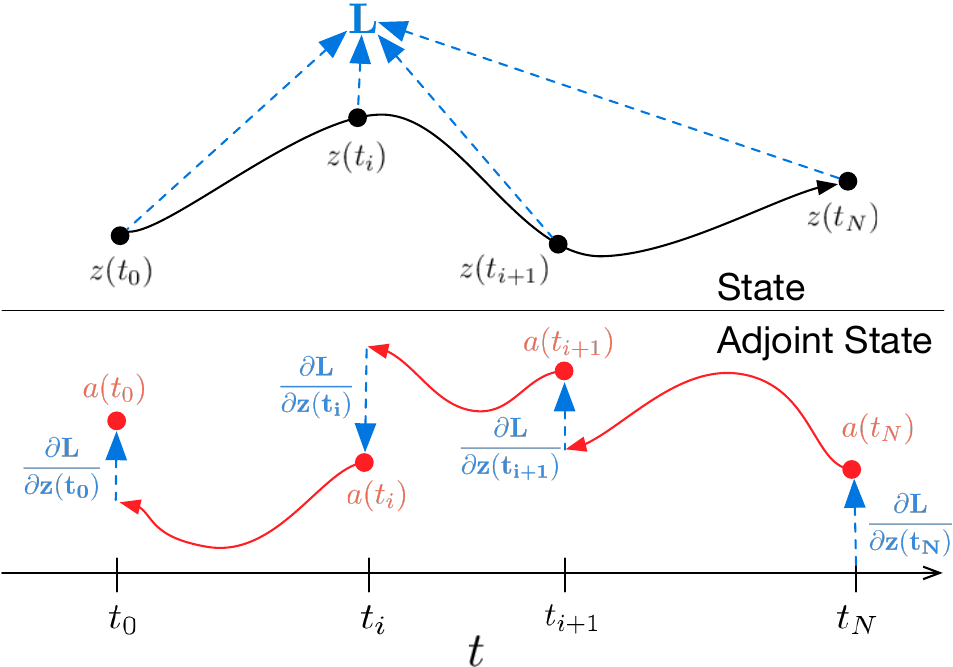}
	\caption{Reverse-mode differentiation of an ODE solution.
	The adjoint sensitivity method solves an augmented ODE backwards in time.
	The augmented system contains both the original state and the sensitivity of the loss with respect to the state.
	If the loss depends directly on the state at multiple observation times, the adjoint state must be updated in the direction of the partial derivative of the loss with respect to each observation.}
	\label{fig:AdjointFig}
\end{wrapfigure}%
To optimize $L$, we require gradients with respect to $\theta$.
The first step is to determine how the gradient of the loss depends on the hidden state $\sol(t)$ at each instant.
This quantity is called the \emph{adjoint} $\adj(t) = \nicefrac{\partial{L}}{\partial \sol(t)}$.
Its dynamics are given by another ODE, which can be thought of as the instantaneous analog of the chain rule:
\begin{align}
\frac{d \adj(t)}{d t} = - \adj(t)\tran{} \frac{\partial f(\sol(t), t, \theta)}{\partial \sol}
\label{eq:adj}
\end{align}
We can compute $\nicefrac{\partial L}{\partial \sol(\tstart)}$ by another call to an ODE solver.
This solver must run backwards, starting from the initial value of $\nicefrac{\partial L}{\partial \sol(\tend)}$.
One complication is that solving this ODE requires knowing the value of $\sol(t)$ along its entire trajectory.
However, we can simply recompute $\sol(t)$ backwards in time together with the adjoint, starting from its final value $\sol(\tend)$.

Computing the gradients with respect to the parameters $\theta$ requires evaluating a third integral, which depends on both $\sol(t)$ and $\adj(t)$:
\begin{align}\label{eq:params_diffeq}
\frac{d L}{d \theta} = -\int^\tstart_\tend \adj(t)\tran{} \frac{\partial f(\sol(t), t, \theta)}{\partial \theta} dt
\end{align}
The vector-Jacobian products $\adj(t)^T\frac{\partial f}{\partial \sol}$ and $\adj(t)^T\frac{\partial f}{\partial \theta}$ in \eqref{eq:adj} and \eqref{eq:params_diffeq} can be efficiently evaluated by automatic differentiation, at a time cost similar to that of evaluating $f$.
All integrals for solving $\sol$, $\adj$ and $\frac{\partial L}{\partial \theta}$ can be computed in a single call to an ODE solver, which concatenates the original state, the adjoint, and the other partial derivatives into a single vector.
Algorithm~\ref{algo1} shows how to construct the necessary dynamics, and call an ODE solver to compute all gradients at once.

\begin{algorithm}[h]
\caption{Reverse-mode derivative of an ODE initial value problem}
\label{algo1}
\begin{algorithmic}
\Require dynamics parameters $\theta$, start time $\tstart$, stop time $\tend$, final state $\sol(\tend)$, loss gradient $\nicefrac{\partial L}{\partial \sol(\tend)}$
\State $s_0 = [\sol(\tend), \frac{\partial L}{\partial \sol(\tend)}, {\vzero}_{|\theta|}]$ \Comment{Define initial augmented state}
\Function{\textnormal{aug\_dynamics}}{$[\sol(t), \adj(t), \cdot], t, \theta$}: \Comment{Define dynamics on augmented state}
\State \textbf{return} $[f(\sol(t), t, \theta), -\adj(t)\tran{} \frac{\partial f}{\partial \sol}, -\adj(t)\tran{} \frac{\partial f}{\partial \theta}]$ \Comment{Compute vector-Jacobian products}
\EndFunction %
\State $[\sol(t_0), \frac{\partial L}{\partial \sol(t_0)}, \frac{\partial L}{\partial \theta}] = \solvefunc(s_0, \textnormal{aug\_dynamics}, \tend, \tstart, \theta)$ \Comment{Solve reverse-time ODE}
\Ensure $\frac{\partial L}{\partial \sol(t_0)}, \frac{\partial L}{\partial \theta}$ \Comment{Return gradients}
\end{algorithmic}
\end{algorithm}

Most ODE solvers have the option to output the state $\sol(t)$ at multiple times.
When the loss depends on these intermediate states, the reverse-mode derivative must be broken into a sequence of separate solves, one between each consecutive pair of output times (Figure~\ref{fig:AdjointFig}).
At each observation, the adjoint must be adjusted in the direction of the corresponding partial derivative $\nicefrac{\partial L}{\partial \sol(t_i)}$.

The results above extend those of \citet[section 2.4.2]{stapor2018optimization}.
An extended version of Algorithm~\ref{algo1} including derivatives w.r.t.\ $\tstart$ and $\tend$ can be found in Appendix~\ref{sec:full-alg}.
Detailed derivations are provided in Appendix~\ref{sec:modern_adj_proof}.
Appendix~\ref{autograd code} provides Python code which computes all derivatives for \texttt{scipy.integrate.odeint} by extending the \href{https://www.github.com/HIPS/autograd}{\texttt{autograd}} automatic differentiation package.
This code also supports all higher-order derivatives.
We have since released a PyTorch~\citep{paszke2017automatic} implementation, including GPU-based implementations of several standard ODE solvers at \href{https://github.com/rtqichen/torchdiffeq}{\texttt{github.com/rtqichen/torchdiffeq}}.

\section{Replacing residual networks with ODEs for supervised learning}
\label{sec:odenets}
In this section, we experimentally investigate the training of neural ODEs for supervised learning.

\paragraph{Software}
To solve ODE initial value problems numerically, we use the implicit Adams method implemented in LSODE and VODE and interfaced through the \texttt{scipy.integrate} package.
Being an implicit method, it has better guarantees than explicit methods such as Runge-Kutta but requires solving a nonlinear optimization problem at every step.
This setup makes direct backpropagation through the integrator difficult.
We implement the adjoint sensitivity method in Python's \texttt{autograd} framework~\citep{maclaurin2015autograd}.
For the experiments in this section, we evaluated the hidden state dynamics and their derivatives on the GPU using TensorFlow, which were then called from the Fortran ODE solvers, which were called from Python \texttt{autograd} code.

\paragraph{Model Architectures}
\begin{wraptable}[8]{r}{0.5\linewidth}
	\vspace{-6mm}
	\caption{Performance on MNIST. $^\dagger$From \citet{lecun1998gradient}.}
	\label{tab:odenet}
	\centering
	\scalebox{0.8}{
		\begin{tabular}{lcccc}
			\toprule
			\multicolumn{1}{l}{} & Test Error & \# Params & Memory & Time \\
			\midrule
			\multicolumn{1}{l}{1-Layer MLP$^\dagger$} & 1.60\% & 0.24 M & - & - \\
			\multicolumn{1}{l}{ResNet} & 0.41\% & 0.60 M & $\mathcal{O}(L)$ & $\mathcal{O}(L)$ \\
			\multicolumn{1}{l}{RK-Net} & 0.47\% & 0.22 M & $\mathcal{O}(\tilde{L})$ & $\mathcal{O}(\tilde{L})$ \\
			\multicolumn{1}{l}{ODE-Net} & 0.42\% & 0.22 M & $\boldsymbol{\mathcal{O}(1)}$ & $\mathcal{O}(\tilde{L})$ \\
			\bottomrule
		\end{tabular}
	}
\end{wraptable}
We experiment with a small residual network which downsamples the input twice then applies 6 standard residual blocks~\citep{he2016identity}, which are replaced by an ODESolve module in the ODE-Net variant.
We also test a network with the same architecture but where gradients are backpropagated directly through a Runge-Kutta integrator, referred to as RK-Net.
Table~\ref{tab:odenet} shows test error, number of parameters, and memory cost. $L$ denotes the number of layers in the ResNet, and $\tilde{L}$ is the number of function evaluations that the ODE solver requests in a single forward pass, which can be interpreted as an implicit number of layers.
We find that ODE-Nets and RK-Nets can achieve around the same performance as the ResNet.

\paragraph{Error Control in ODE-Nets}
ODE solvers can approximately ensure that the output is within a given tolerance of the true solution.
Changing this tolerance changes the behavior of the network.
We first verify that error can indeed be controlled in Figure~\ref{fig:odenet-plots}a.
The time spent by the forward call is proportional to the number of function evaluations (Figure~\ref{fig:odenet-plots}b), so tuning the tolerance gives us a trade-off between accuracy and computational cost.
One could train with high accuracy, but switch to a lower accuracy at test time.

\begin{figure}[h]
	\centering
	\scalebox{1.0}{
		\begin{subfigure}[b]{0.255\linewidth}
			\includegraphics[width=\linewidth, trim=10px 10px 0 0, clip]{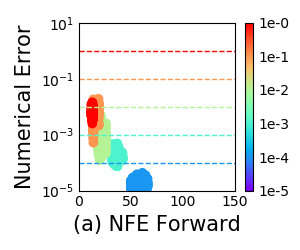}
		\end{subfigure}%
		\begin{subfigure}[b]{0.255\linewidth}
			\includegraphics[width=\linewidth, trim=10px 10px 0 0, clip]{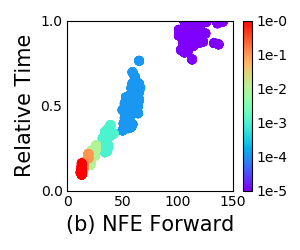}
		\end{subfigure}%
		\begin{subfigure}[b]{0.255\linewidth}
			\includegraphics[width=\linewidth]{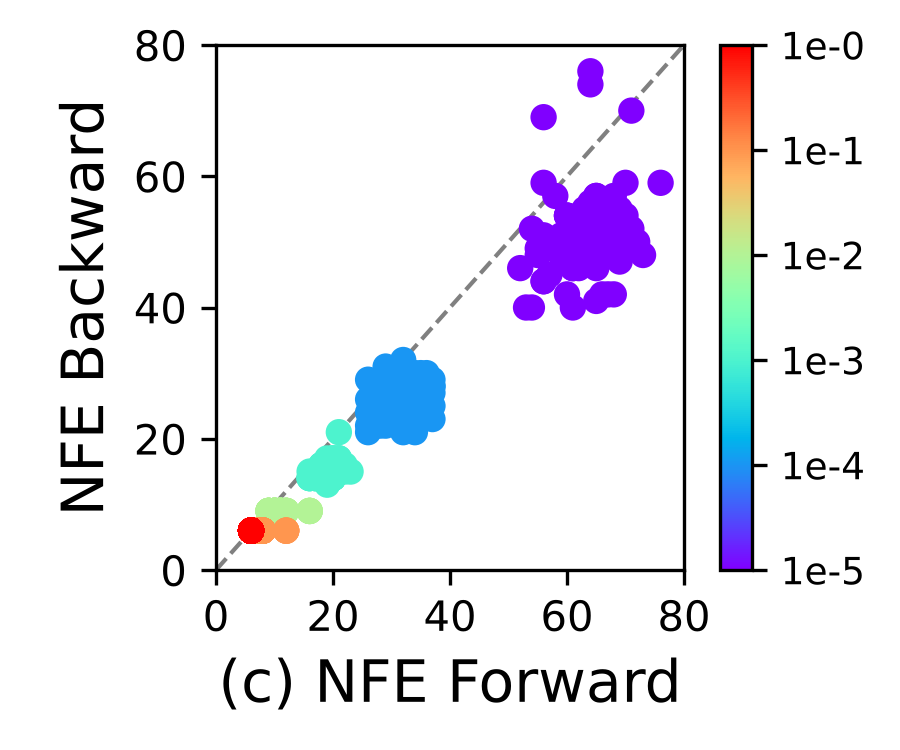}
		\end{subfigure}%
		\begin{subfigure}[b]{0.235\linewidth}
			\includegraphics[width=\linewidth, trim=9px 10px 0 0, clip]{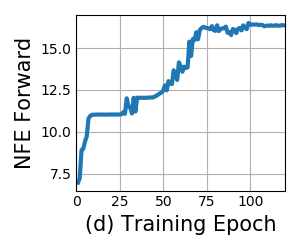}
		\end{subfigure}
	}
	\caption{Statistics of a trained ODE-Net. (NFE = number of function evaluations.)
	Panel (c) has been corrected from earlier versions of this paper; see text.}
	\label{fig:odenet-plots}
\end{figure}

Figure~\ref{fig:odenet-plots}c shows that, for this model, the number of evaluations in the backward pass is somewhat smaller than in the forward pass, with a median ratio of about 0.8 for tolerances between \texttt{1e-5} and \texttt{1e-1}.\footnote{\textbf{Correction:} Earlier versions of this paper reported that the backward pass required roughly half as many evaluations as the forward pass, and called this a surprising result. \citet{vanni2026replication} found that this was caused by a bug in our experiment logging. Figure~\ref{fig:odenet-plots}c has been corrected. We thank them for finding and diagnosing this error.}
This suggests that, in this setting, the adjoint sensitivity method can be not only more memory efficient, but also slightly more computationally efficient than directly backpropagating through the integrator, because the latter approach will need to backprop through each function evaluation in the forward pass.

\paragraph{Network Depth}
It's not clear how to define the `depth' of an ODE solution.
A related quantity is the number of evaluations of the hidden state dynamics required, a detail delegated to the ODE solver and dependent on the initial state or input.
Figure~\ref{fig:odenet-plots}d shows that the number of function evaluations increases throughout training, presumably adapting to increasing complexity of the model.

\section{Continuous Normalizing Flows}
\label{sec:cnf}

The discretized equation~\eqref{eq:res} also appears in normalizing flows~\citep{rezende2015variational} and the NICE framework~\citep{dinh2014nice}.
These methods use the change of variables theorem to compute exact changes in probability if samples are transformed through a bijective function $f$:
\begin{align}
{\cnfx}_1 = f({\cnfx}_0)\; \implies \log p({\cnfx}_1) = \log p({\cnfx}_0) - \log \left| \det \frac{\partial f}{\partial {\cnfx}_0} \right|
\label{eq:change of variables}
\end{align}
An example is the planar normalizing flow~\citep{rezende2015variational}:
\begin{align}
{\cnfx}(t+1) = {\cnfx}(t) + uh(w\tran{} {\cnfx}(t) + b),\quad \log p({\cnfx}(t+1)) = \log p({\cnfx}(t)) - \log \left| 1 +  u\tran{} \frac{\partial h}{\partial {\cnfx}} \right|
\end{align}

Generally, the main bottleneck to using the change of variables formula is computing the determinant of the Jacobian $\nicefrac{\partial f}{\partial {\cnfx}}$, which has a cubic cost in either the dimension of $\cnfx$, or the number of hidden units.
Recent work explores the tradeoff between the expressiveness of normalizing flow layers and computational cost~\citep{kingma2016improved,tomczak2016improving,berg2018sylvester}.

Surprisingly, moving from a discrete set of layers to a continuous transformation simplifies the computation of the change in normalizing constant:

\begin{theorem}[Instantaneous Change of Variables]
Let ${\cnfx}(t)$ be a finite continuous random variable with probability $p({\cnfx}(t))$ dependent on time.
Let $\frac{d{\cnfx}}{dt} = f({\cnfx}(t), t)$ be a differential equation describing a continuous-in-time transformation of ${\cnfx}(t)$.
Assuming that $f$ is uniformly Lipschitz continuous in ${\cnfx}$ and continuous in $t$, then the change in log probability also follows a differential equation,
\begin{align}
\frac{\partial \log p({\cnfx}(t))}{\partial t} = -\tr\left(\frac{d f}{d {\cnfx}(t)}\right)
\end{align}
\end{theorem}
Proof in Appendix~\ref{sec:proof}.
Instead of the log determinant in \eqref{eq:change of variables}, we now only require a trace operation. %
Also unlike standard finite flows, the differential equation $f$ does not need to be bijective, since if uniqueness is satisfied, then the entire transformation is automatically bijective.

As an example application of the instantaneous change of variables, we can examine the continuous analog of the planar flow, and its change in normalization constant:
\begin{align}\label{eq:cnf}
\frac{d{\cnfx}(t)}{dt} = uh(w\tran{}{\cnfx}(t) + b),\quad \frac{\partial\log p({\cnfx}(t))}{\partial t} = -u\tran{} \frac{\partial h}{\partial {\cnfx}(t)}
\end{align}
Given an initial distribution $p({\cnfx}(0))$, we can sample from $p({\cnfx}(t))$ and evaluate its density by solving this combined ODE.

\paragraph{Using multiple hidden units with linear cost}
While $\det$ is not a linear function, the trace function is, which implies $\tr(\sum_n J_n) = \sum_n \tr(J_n)$.
Thus if our dynamics is given by a sum of functions then the differential equation for the log density is also a sum:
\begin{align}
\frac{d{\cnfx}(t)}{dt} = \sum_{n=1}^M f_n({\cnfx}(t)),\quad \frac{d\log p({\cnfx}(t))}{dt} = \sum_{n=1}^M \tr\left(\frac{\partial f_n}{\partial {\cnfx}}\right)
\end{align}
This means we can cheaply evaluate flow models having many hidden units, with a cost only linear in the number of hidden units $M$.
Evaluating such `wide' flow layers using standard normalizing flows costs $\mathcal{O}(M^3)$, meaning that standard NF architectures use many layers of only a single hidden unit.

\paragraph{Time-dependent dynamics}
We can specify the parameters of a flow as a function of $t$, making the differential equation $f({\cnfx}(t), t)$ change with $t$.
This parameterization is a kind of hypernetwork~\citep{ha2016hypernetworks}.
We also introduce a gating mechanism for each hidden unit, $\frac{d{\cnfx}}{dt} = \sum_n \sigma_n(t) f_n({\cnfx})$ where $\sigma_n(t) \in (0, 1)$ is a neural network that learns when the dynamic $f_n({\cnfx})$ should be applied.
We call these models continuous normalizing flows (CNF).

\subsection{Experiments with Continuous Normalizing Flows}
\newcommand{\nfwidth}{0.3\linewidth}%
\begin{figure}[b]
	\centering
	\begin{subfigure}[b]{0.05\linewidth}
		\hfill
	\end{subfigure}
	\begin{subfigure}[b]{0.1\linewidth}
		\hfill
	\end{subfigure}%
	\begin{subfigure}[b]{\nfwidth}
		\hspace{1mm} K=2 \hspace{0.12\linewidth} K=8 \hspace{0.12\linewidth} K=32 \hfill
	\end{subfigure}%
	\begin{subfigure}[b]{\nfwidth}
		\hspace{1mm} M=2 \hspace{0.12\linewidth} M=8 \hspace{0.08\linewidth} M=32 \hfill
	\end{subfigure}
	\begin{subfigure}[b]{0.15\linewidth}
		\hfill
	\end{subfigure}\\
	\begin{subfigure}[b]{0.01\linewidth}
		\centering
		$1$
		\vspace{1.7mm}
	\end{subfigure}
	\begin{subfigure}[b]{0.1\linewidth}
		\centering
		\includegraphics[width=0.9\linewidth]{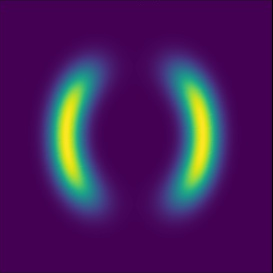}
	\end{subfigure}%
	\begin{subfigure}[b]{\nfwidth}
		\centering
		\includegraphics[width=\nfwidth]{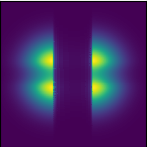}
		\includegraphics[width=\nfwidth]{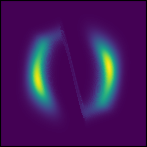}
		\includegraphics[width=\nfwidth]{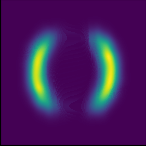}
	\end{subfigure}%
	\begin{subfigure}[b]{\nfwidth}
		\centering
		\includegraphics[width=\nfwidth]{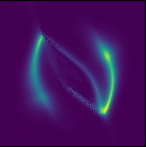}
		\includegraphics[width=\nfwidth]{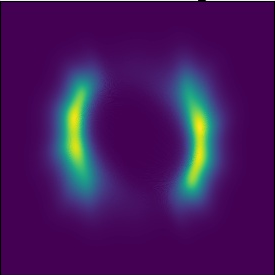}
		\includegraphics[width=\nfwidth]{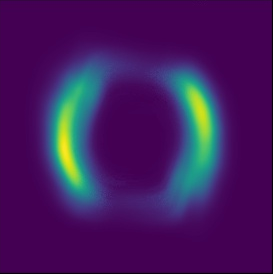}
	\end{subfigure}%
	\begin{subfigure}[b]{0.15\linewidth}
		\centering
		\includegraphics[width=0.97\linewidth, trim=10px 10px 0 0px, clip]{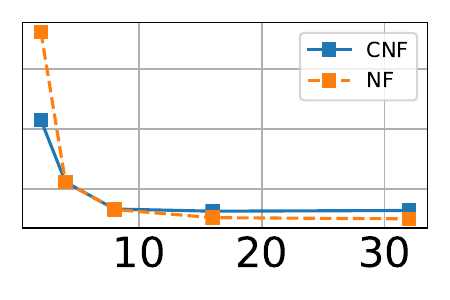}
	\end{subfigure}\\
	\begin{subfigure}[b]{0.01\linewidth}
		\centering
		$2$
		\vspace{1.7mm}
	\end{subfigure}
	\begin{subfigure}[b]{0.1\linewidth}
		\centering
		\includegraphics[width=0.9\linewidth]{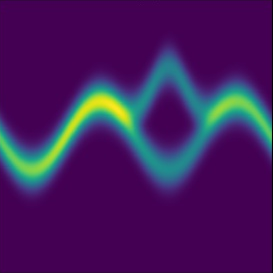}
	\end{subfigure}%
	\begin{subfigure}[b]{\nfwidth}
		\centering
		\includegraphics[width=\nfwidth]{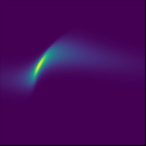}
		\includegraphics[width=\nfwidth]{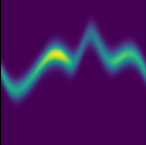}
		\includegraphics[width=\nfwidth]{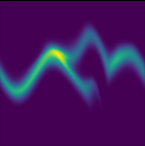}
	\end{subfigure}%
	\begin{subfigure}[b]{\nfwidth}
		\centering
		\includegraphics[width=\nfwidth]{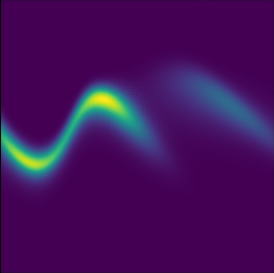}
		\includegraphics[width=\nfwidth]{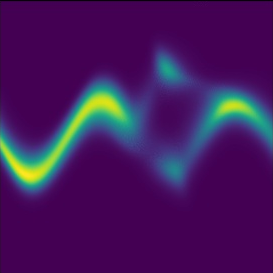}
		\includegraphics[width=\nfwidth]{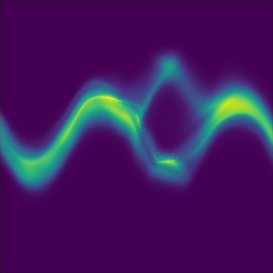}
	\end{subfigure}%
	\begin{subfigure}[b]{0.15\linewidth}
		\centering
		\includegraphics[width=0.97\linewidth, trim=10px 10px 0 0px, clip]{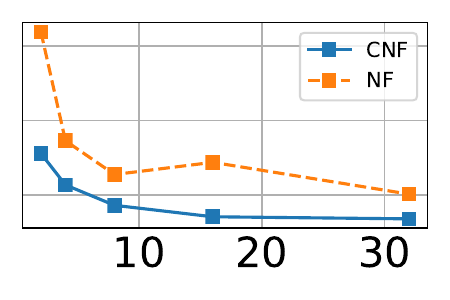}
	\end{subfigure}\\
	\begin{subfigure}[b]{0.01\linewidth}
		\centering
		$3$
		\vspace{1.7mm}
		\caption*{}
	\end{subfigure}
	\begin{subfigure}[b]{0.1\linewidth}
		\centering
		\includegraphics[width=0.9\linewidth]{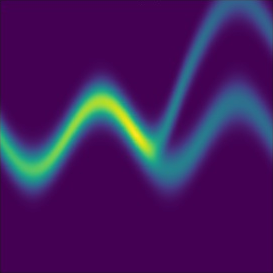}
		\caption{Target}
	\end{subfigure}%
	\begin{subfigure}[b]{\nfwidth}
		\centering
		\includegraphics[width=\nfwidth]{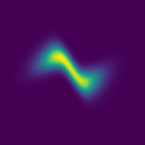}
		\includegraphics[width=\nfwidth]{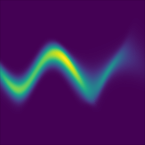}
		\includegraphics[width=\nfwidth]{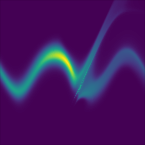}
		\caption{NF}
	\end{subfigure}%
	\begin{subfigure}[b]{\nfwidth}
		\centering
		\includegraphics[width=\nfwidth]{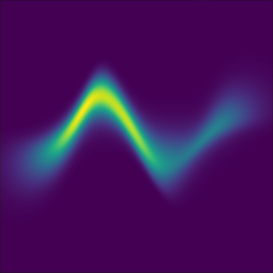}
		\includegraphics[width=\nfwidth]{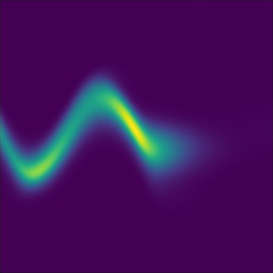}
		\includegraphics[width=\nfwidth]{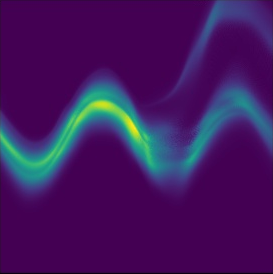}
		\caption{CNF}
	\end{subfigure}%
	\begin{subfigure}[b]{0.15\linewidth}
		\centering
		\includegraphics[width=0.97\linewidth, trim=10px 10px 0 0px, clip]{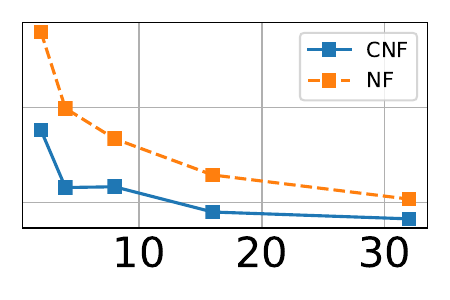}
		\caption{Loss vs. K/M}
	\end{subfigure}
	\caption{Comparison of normalizing flows versus continuous normalizing flows.
	The model capacity of normalizing flows is determined by their depth (K), while continuous normalizing flows can also increase capacity by increasing width (M), making them easier to train.}
\label{fig:cnf_toy}
\end{figure}

We first compare continuous and discrete planar flows at learning to sample from a known distribution.
We show that a planar CNF with $M$ hidden units can be at least as expressive as a planar NF with $K=M$ layers, and sometimes much more expressive.

\paragraph{Density matching}
We configure the CNF as described above, and train for 10,000 iterations using Adam~\citep{kingma2014adam}.
In contrast, the NF is trained for 500,000 iterations using RMSprop~\citep{hinton2012neural}, as suggested by \citet{rezende2015variational}.
For this task, we minimize $\KL{q(\vx)}{p(\vx)}$ as the loss function where $q$ is the flow model and the target density $p(\cdot)$ can be evaluated.
Figure~\ref{fig:cnf_toy} shows that CNF generally achieves lower loss.

\paragraph{Maximum Likelihood Training}
\begin{figure}
	\begin{subfigure}[b]{0.49\linewidth}
		\centering
		\begin{subfigure}[b]{0.03\linewidth}
			\rotatebox[origin=t]{90}{\scriptsize Density}\vspace{0.4\linewidth}
		\end{subfigure}%
		\begin{subfigure}[b]{0.136\linewidth}
			\caption*{5\%}
			\vspace*{-2.5mm}
			\includegraphics[width=\linewidth]{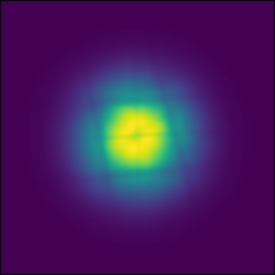}
		\end{subfigure}%
		\begin{subfigure}[b]{0.136\linewidth}
			\caption*{20\%}
			\vspace*{-2.5mm}
			\includegraphics[width=\linewidth]{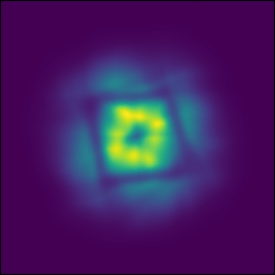}
		\end{subfigure}%
		\begin{subfigure}[b]{0.136\linewidth}
			\caption*{40\%}
			\vspace*{-2.5mm}
			\includegraphics[width=\linewidth]{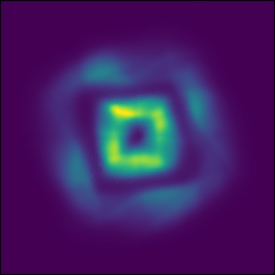}
		\end{subfigure}%
		\begin{subfigure}[b]{0.136\linewidth}
			\caption*{60\%}
			\vspace*{-2.5mm}
			\includegraphics[width=\linewidth]{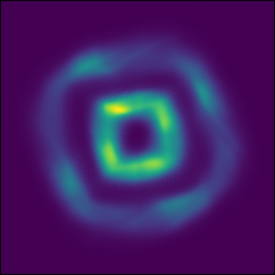}
		\end{subfigure}%
		\begin{subfigure}[b]{0.136\linewidth}
			\caption*{80\%}
			\vspace*{-2.5mm}
			\includegraphics[width=\linewidth]{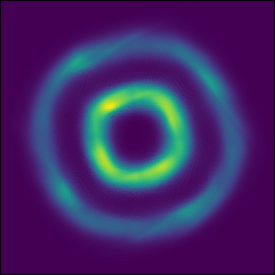}
		\end{subfigure}%
		\begin{subfigure}[b]{0.136\linewidth}
			\caption*{100\%}
			\vspace*{-2.5mm}
			\includegraphics[width=\linewidth]{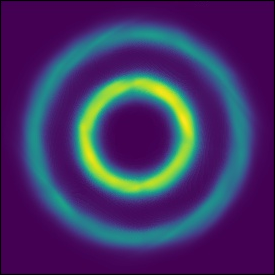}
		\end{subfigure}
		\begin{subfigure}[b]{0.136\linewidth}
			\hfill
		\end{subfigure}\\
		\begin{subfigure}[b]{0.03\linewidth}
			\rotatebox[origin=t]{90}{\scriptsize Samples}\vspace{0.2\linewidth}
		\end{subfigure}%
		\begin{subfigure}[b]{0.136\linewidth}
			\includegraphics[width=\linewidth]{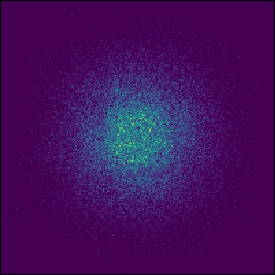}
		\end{subfigure}%
		\begin{subfigure}[b]{0.136\linewidth}
			\includegraphics[width=\linewidth]{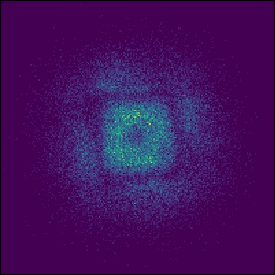}
		\end{subfigure}%
		\begin{subfigure}[b]{0.136\linewidth}
			\includegraphics[width=\linewidth]{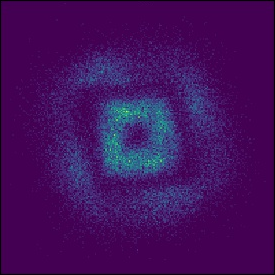}
		\end{subfigure}%
		\begin{subfigure}[b]{0.136\linewidth}
			\includegraphics[width=\linewidth]{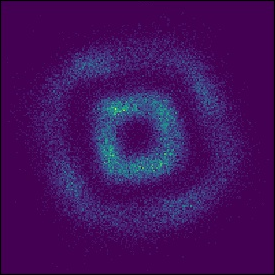}
		\end{subfigure}%
		\begin{subfigure}[b]{0.136\linewidth}
			\includegraphics[width=\linewidth]{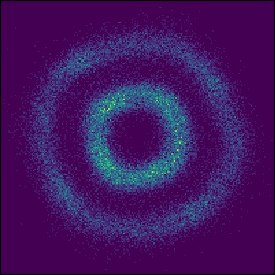}
		\end{subfigure}%
		\begin{subfigure}[b]{0.136\linewidth}
			\includegraphics[width=\linewidth]{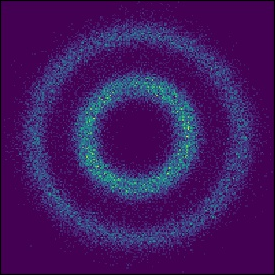}
		\end{subfigure}
		\begin{subfigure}[b]{0.136\linewidth}
			\includegraphics[width=\linewidth]{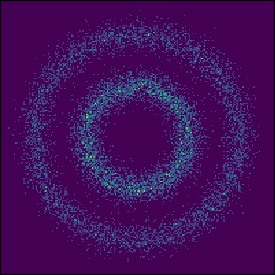}
		\end{subfigure}\\
		\begin{subfigure}[b]{0.03\linewidth}
			\rotatebox[origin=t]{90}{\footnotesize NF}\vspace{1.4\linewidth}
		\end{subfigure}%
		\begin{subfigure}[b]{0.136\linewidth}
			\includegraphics[width=\linewidth]{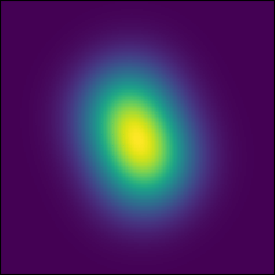}
		\end{subfigure}%
		\begin{subfigure}[b]{0.136\linewidth}
			\includegraphics[width=\linewidth]{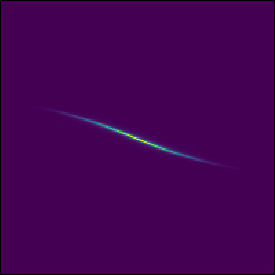}
		\end{subfigure}%
		\begin{subfigure}[b]{0.136\linewidth}
			\includegraphics[width=\linewidth]{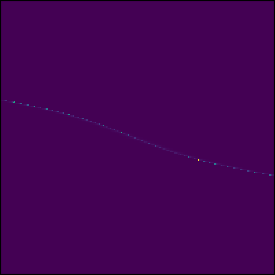}
		\end{subfigure}%
		\begin{subfigure}[b]{0.136\linewidth}
			\includegraphics[width=\linewidth]{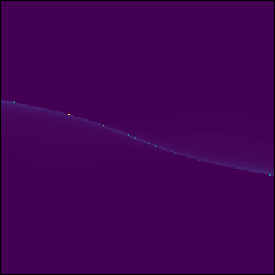}
		\end{subfigure}%
		\begin{subfigure}[b]{0.136\linewidth}
			\includegraphics[width=\linewidth]{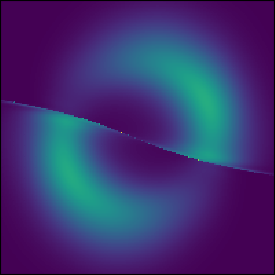}
		\end{subfigure}%
		\begin{subfigure}[b]{0.136\linewidth}
			\includegraphics[width=\linewidth]{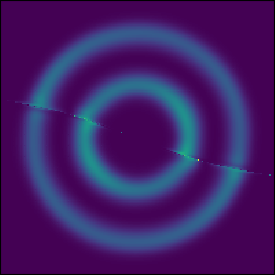}
		\end{subfigure}
		\begin{subfigure}[b]{0.136\linewidth}
			\hfill {\footnotesize Target\vspace{0.6\linewidth}} \hfill
		\end{subfigure}\\
		\caption{Two Circles}
		\label{fig:2circles}
	\end{subfigure}\hfill
	\begin{subfigure}[b]{0.49\linewidth}
		\centering
		\begin{subfigure}[b]{0.03\linewidth}
			\caption*{}
			\rotatebox[origin=t]{90}{\scriptsize Density}\vspace{0.4\linewidth}
		\end{subfigure}%
		\begin{subfigure}[b]{0.136\linewidth}
			\caption*{5\%}
			\vspace*{-2.5mm}
			\includegraphics[width=\linewidth]{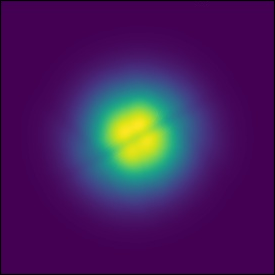}
		\end{subfigure}%
		\begin{subfigure}[b]{0.136\linewidth}
			\caption*{20\%}
			\vspace*{-2.5mm}
			\includegraphics[width=\linewidth]{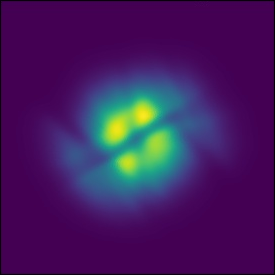}
		\end{subfigure}%
		\begin{subfigure}[b]{0.136\linewidth}
			\caption*{40\%}
			\vspace*{-2.5mm}
			\includegraphics[width=\linewidth]{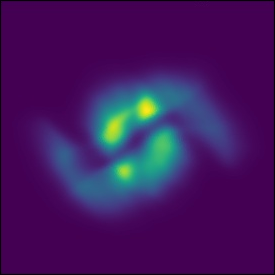}
		\end{subfigure}%
		\begin{subfigure}[b]{0.136\linewidth}
			\caption*{60\%}
			\vspace*{-2.5mm}
			\includegraphics[width=\linewidth]{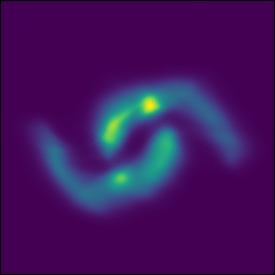}
		\end{subfigure}%
		\begin{subfigure}[b]{0.136\linewidth}
			\caption*{80\%}
			\vspace*{-2.5mm}
			\includegraphics[width=\linewidth]{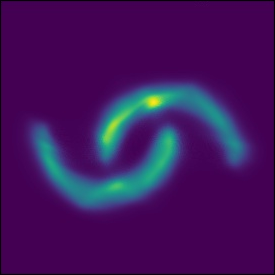}
		\end{subfigure}%
		\begin{subfigure}[b]{0.136\linewidth}
			\caption*{100\%}
			\vspace*{-2.5mm}
			\includegraphics[width=\linewidth]{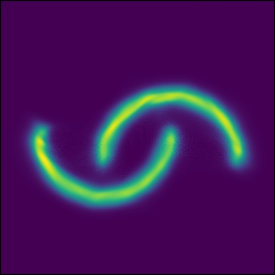}
		\end{subfigure}
		\begin{subfigure}[b]{0.136\linewidth}
			\hfill
		\end{subfigure}\\
		\begin{subfigure}[b]{0.03\linewidth}
			\rotatebox[origin=t]{90}{\scriptsize Samples}\vspace{0.2\linewidth}
		\end{subfigure}%
		\begin{subfigure}[b]{0.136\linewidth}
			\includegraphics[width=\linewidth]{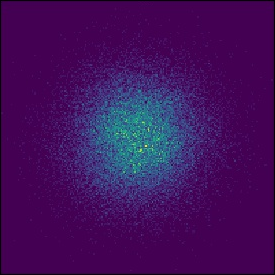}
		\end{subfigure}%
		\begin{subfigure}[b]{0.136\linewidth}
			\includegraphics[width=\linewidth]{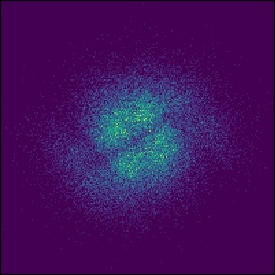}
		\end{subfigure}%
		\begin{subfigure}[b]{0.136\linewidth}
			\includegraphics[width=\linewidth]{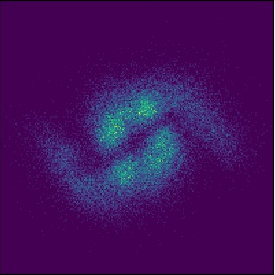}
		\end{subfigure}%
		\begin{subfigure}[b]{0.136\linewidth}
			\includegraphics[width=\linewidth]{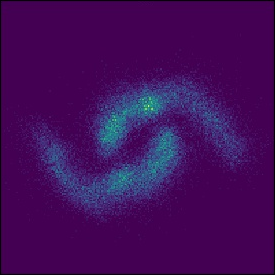}
		\end{subfigure}%
		\begin{subfigure}[b]{0.136\linewidth}
			\includegraphics[width=\linewidth]{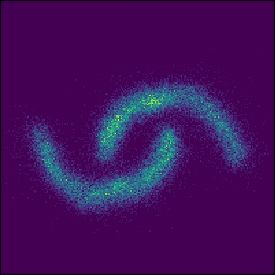}
		\end{subfigure}%
		\begin{subfigure}[b]{0.136\linewidth}
			\includegraphics[width=\linewidth]{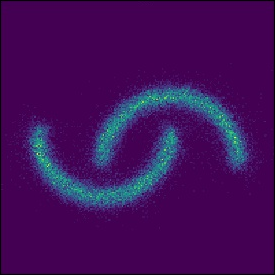}
		\end{subfigure}
		\begin{subfigure}[b]{0.136\linewidth}
			\includegraphics[width=\linewidth]{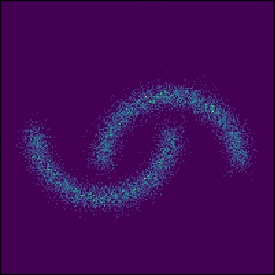}
		\end{subfigure}\\
		\begin{subfigure}[b]{0.03\linewidth}
			\rotatebox[origin=t]{90}{\footnotesize NF}\vspace{1.4\linewidth}
		\end{subfigure}%
		\begin{subfigure}[b]{0.136\linewidth}
			\includegraphics[width=\linewidth]{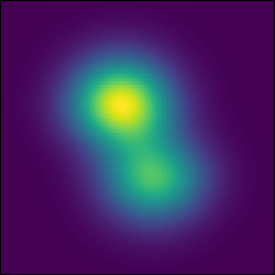}
		\end{subfigure}%
		\begin{subfigure}[b]{0.136\linewidth}
			\includegraphics[width=\linewidth]{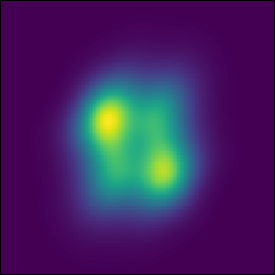}
		\end{subfigure}%
		\begin{subfigure}[b]{0.136\linewidth}
			\includegraphics[width=\linewidth]{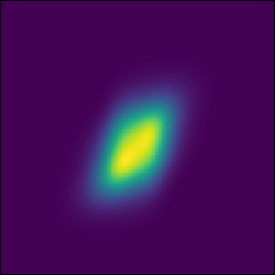}
		\end{subfigure}%
		\begin{subfigure}[b]{0.136\linewidth}
			\includegraphics[width=\linewidth]{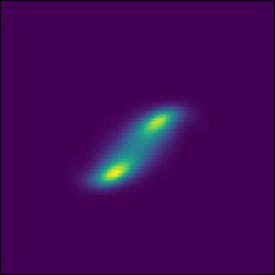}
		\end{subfigure}%
		\begin{subfigure}[b]{0.136\linewidth}
			\includegraphics[width=\linewidth]{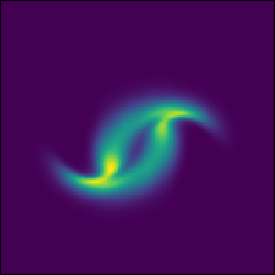}
		\end{subfigure}%
		\begin{subfigure}[b]{0.136\linewidth}
			\includegraphics[width=\linewidth]{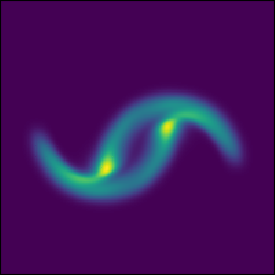}
		\end{subfigure}
		\begin{subfigure}[b]{0.136\linewidth}
			\hfill {\footnotesize Target\vspace{0.6\linewidth}} \hfill
		\end{subfigure}\\
		\caption{Two Moons}
		\label{fig:2moons}
	\end{subfigure}
\caption{\textbf{Visualizing the transformation from noise to data.}
Continuous-time normalizing flows are reversible, so we can train on a density estimation task and still be able to sample from the learned density efficiently. %
}
\label{fig:cnf_mle}
\end{figure}
A useful property of continuous-time normalizing flows is that we can compute the reverse transformation for about the same cost as the forward pass, which cannot be said for normalizing flows.
This lets us train the flow on a density estimation task by performing maximum likelihood estimation, which maximizes $\E_{p(\vx)}[\log q(\vx)]$ where $q(\cdot)$ is computed using the appropriate change of variables theorem, then afterwards reverse the CNF to generate random samples from $q(\vx)$.

For this task, we use 64 hidden units for CNF, and 64 stacked one-hidden-unit layers for NF.
Figure~\ref{fig:cnf_mle} shows the learned dynamics.
Instead of showing the initial Gaussian distribution, we display the transformed distribution after a small amount of time which shows the locations of the initial planar flows.
Interestingly, to fit the Two Circles distribution, the CNF rotates the planar flows so that the particles can be evenly spread into circles.
While the CNF transformations are smooth and interpretable, we find that NF transformations are very unintuitive and this model has difficulty fitting the two moons dataset in Figure~\ref{fig:2moons}.

\section{A generative latent function time-series model}
\label{sec:lode}

Applying neural networks to irregularly-sampled data such as medical records, network traffic, or neural spiking data is difficult.
Typically, observations are put into bins of fixed duration, and the latent dynamics are discretized in the same way.
This leads to difficulties with missing data and ill-defined latent variables. %
Missing data can be addressed using generative time-series models~\citep{alvarez2011computationally, futoma_hariharan_heller_2017, mei2017neural, soleimani2017scalable} or data imputation~\citep{che_purushotham_cho_sontag_liu_2016}.
Another approach concatenates time-stamp information to the input of an RNN \citep{pmlr-v56-Choi16, pmlr-v56-Lipton16, du2016recurrent, li2017time}.

We present a continuous-time, generative approach to modeling time series.
Our model represents each time series by a latent trajectory.
Each trajectory is determined from a local initial state, $\sol_{t_0}$, and a global set of latent dynamics shared across all time series.
Given observation times $t_0, t_1, \dots, t_N$ and an initial state $\sol_{t_0}$, an ODE solver produces $\sol_{t_1},  \dots,\sol_{t_N}$, which describe the latent state at each observation.%
We define this generative model formally through a sampling procedure:
\begin{align}
\sol_{t_0} & \sim p(\sol_{t_0}) \\
\sol_{t_1}, \sol_{t_2}, \dots, \sol_{t_N} & = \solvefunc(\sol_{t_0}, f, \theta_f, t_0, \dots, t_N) \\
\textnormal{each} \quad \vx_{t_i} & \sim p(\obs | \sol_{t_i}, \theta_\obs)
\end{align}

\begin{figure}[b]
\centering
\includegraphics[width=\textwidth, trim={0.5cm 10cm 0 5.0cm}, clip]{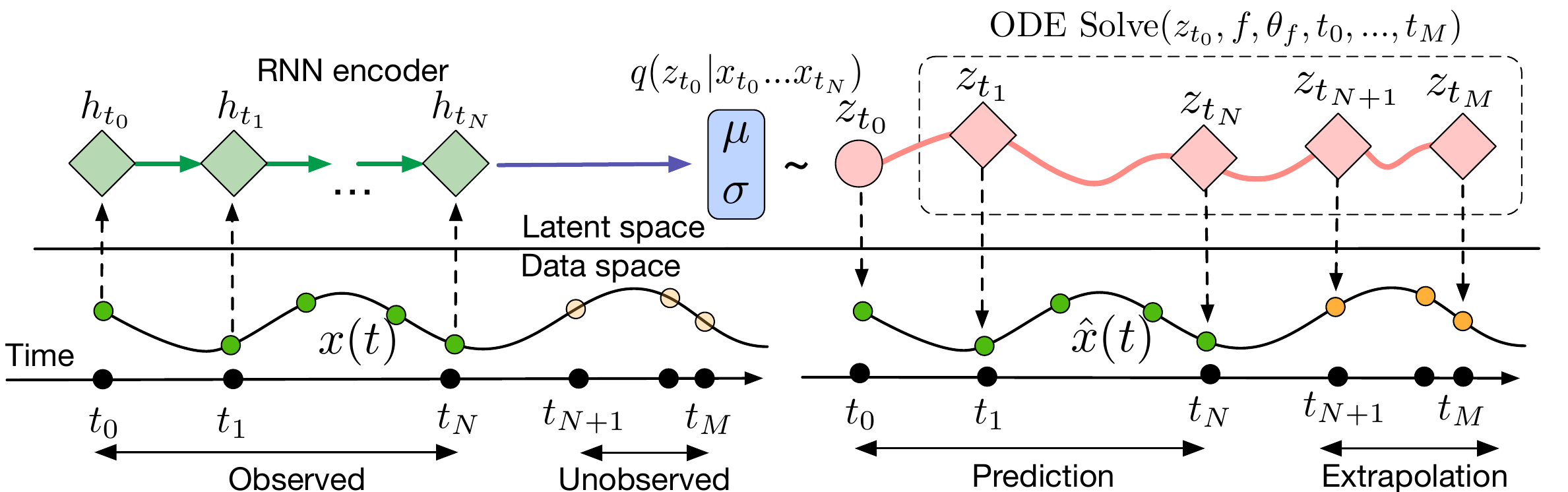}
\caption{Computation graph of the latent ODE model.}
\label{fig:time_series_workflow}
\end{figure}

Function $f$ is a time-invariant function that takes the value $\sol$ at the current time step and outputs the gradient: $\nicefrac{\partial \sol(t)}{\partial t} = f(\sol(t), \theta_f)$.
We parametrize this function using a neural net.
Because $f$ is time-invariant, given any latent state $\sol(t)$, the entire latent trajectory is uniquely defined.
Extrapolating this latent trajectory lets us make predictions arbitrarily far forwards or backwards in time.

\paragraph{Training and Prediction}
We can train this latent-variable model as a variational autoencoder~\citep{kingma2013autoencoding, rezende2014stochastic}, with sequence-valued observations.
Our recognition net is an RNN, which consumes the data sequentially backwards in time, and outputs $q_{\phi}(\latent_0|\vx_1, \vx_2, \dots, \vx_N)$.
A detailed algorithm can be found in Appendix~\ref{lode alg}.
Using ODEs as a generative model allows us to make predictions for arbitrary time points $t_1$...$t_M$ on a continuous timeline.

\begin{wrapfigure}[13]{r}{0.35\textwidth}
\centering
\vspace{-1em} \rotatebox{90}{\hspace{8mm}$\lambda(t)$}\includegraphics[width=.35\textwidth]{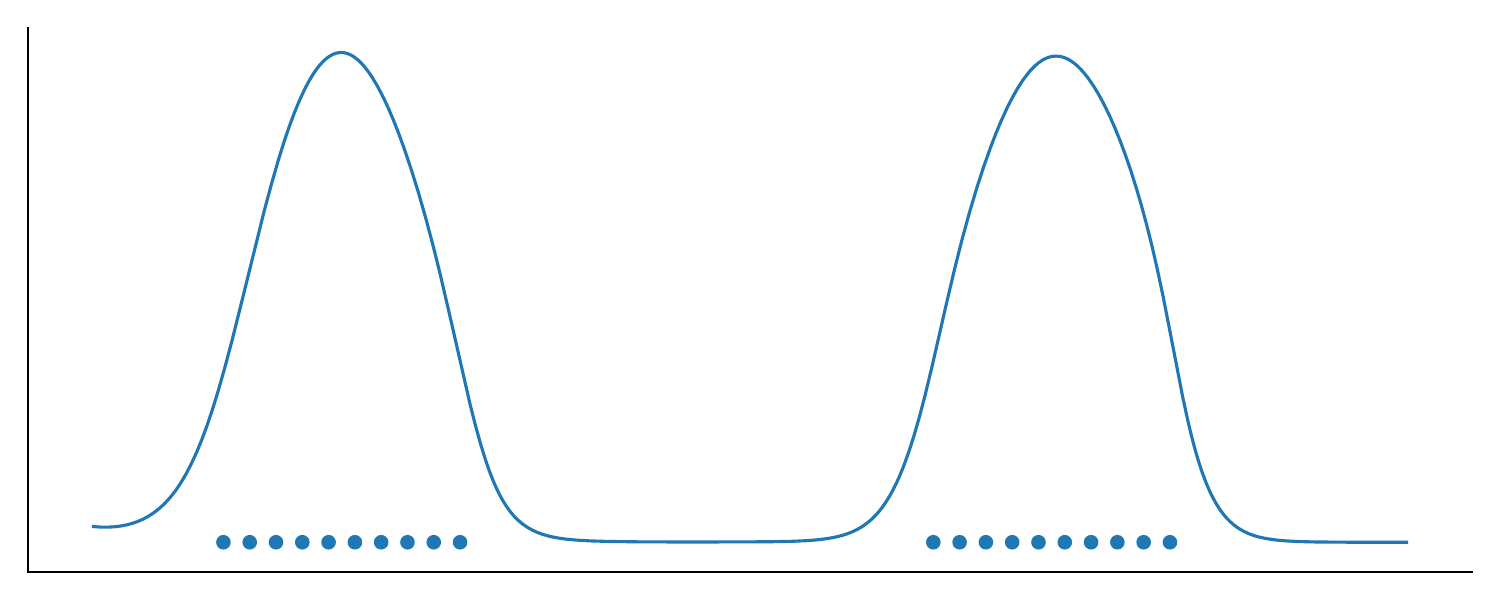}
	$t$
\caption{Fitting a latent ODE dynamics model with a Poisson process likelihood.
Dots show event times.
The line is the learned intensity $\lambda(t)$ of the Poisson process.}
\label{fig:poisson}
\end{wrapfigure}
\paragraph{Poisson Process likelihoods}
The fact that an observation occurred often tells us something about the latent state.
For example, a patient may be more likely to take a medical test if they are sick.
The rate of events can be parameterized by a function of the latent state: $ p(\textnormal{event at time $t$}| \, \sol(t)) = \lograte(\sol(t))$.
Given this rate function, the likelihood of a set of independent observation times in the interval $[t_\textnormal{start}, t_\textnormal{end}]$ is given by an inhomogeneous Poisson process~\citep{palm1943intensitatsschwankungen}:
\begin{align}
\log p(t_1 \dots t_N | \, t_\textnormal{start}, t_\textnormal{end})
= \sum_{i=1}^{N} \log \lograte(\latent(t_i)) - \int_{t_\textnormal{start}}^{ t_\textnormal{end}} \! \lograte(\latent(t)) dt \nonumber
\end{align}
\begin{figure}[t]
	\centering
	\begin{subfigure}[b]{0.33\linewidth}
		\centering
		\includegraphics[trim={0.8cm 0.8cm 0.8cm 1cm},clip,width=0.49\linewidth]{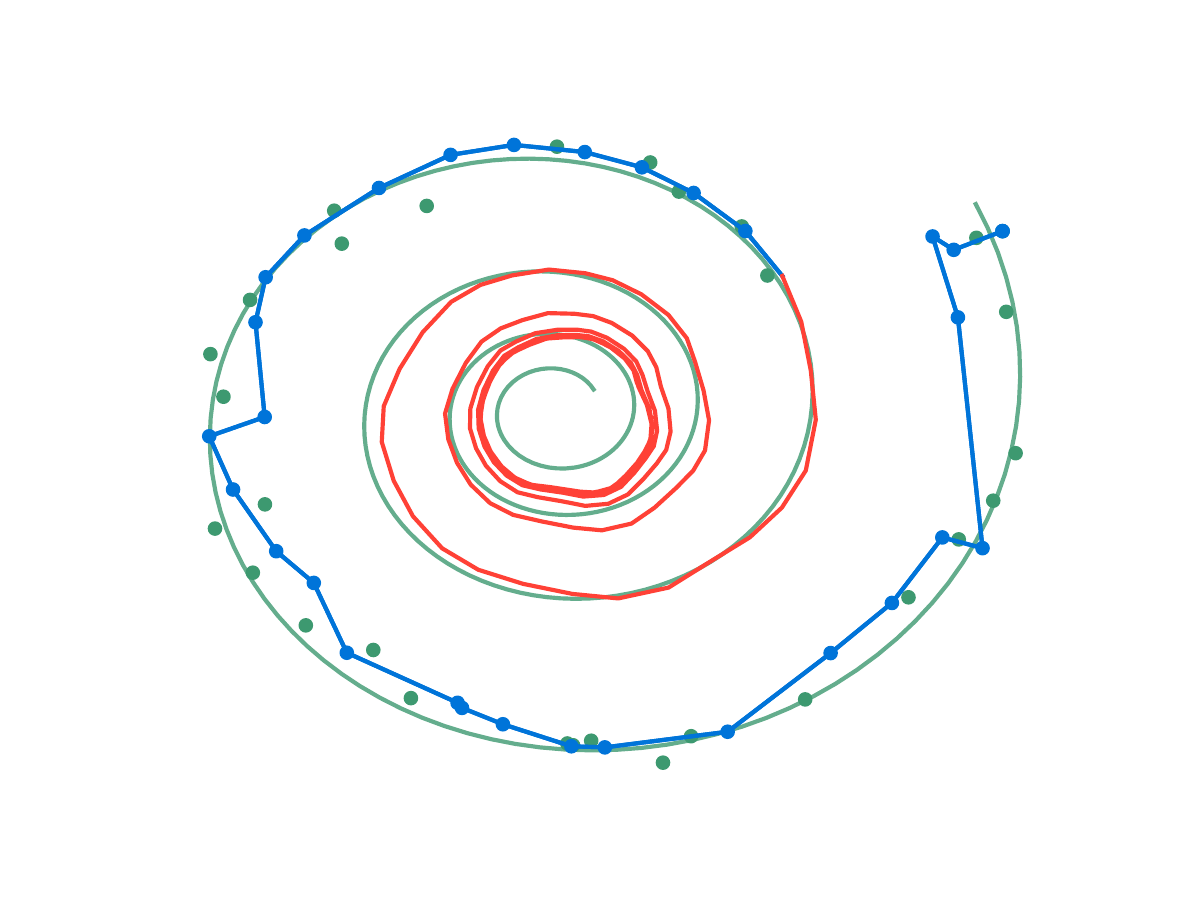}%
		\hfill
		\includegraphics[trim={0.8cm 0.8cm 0.8cm 1cm},clip,width=0.49\linewidth]{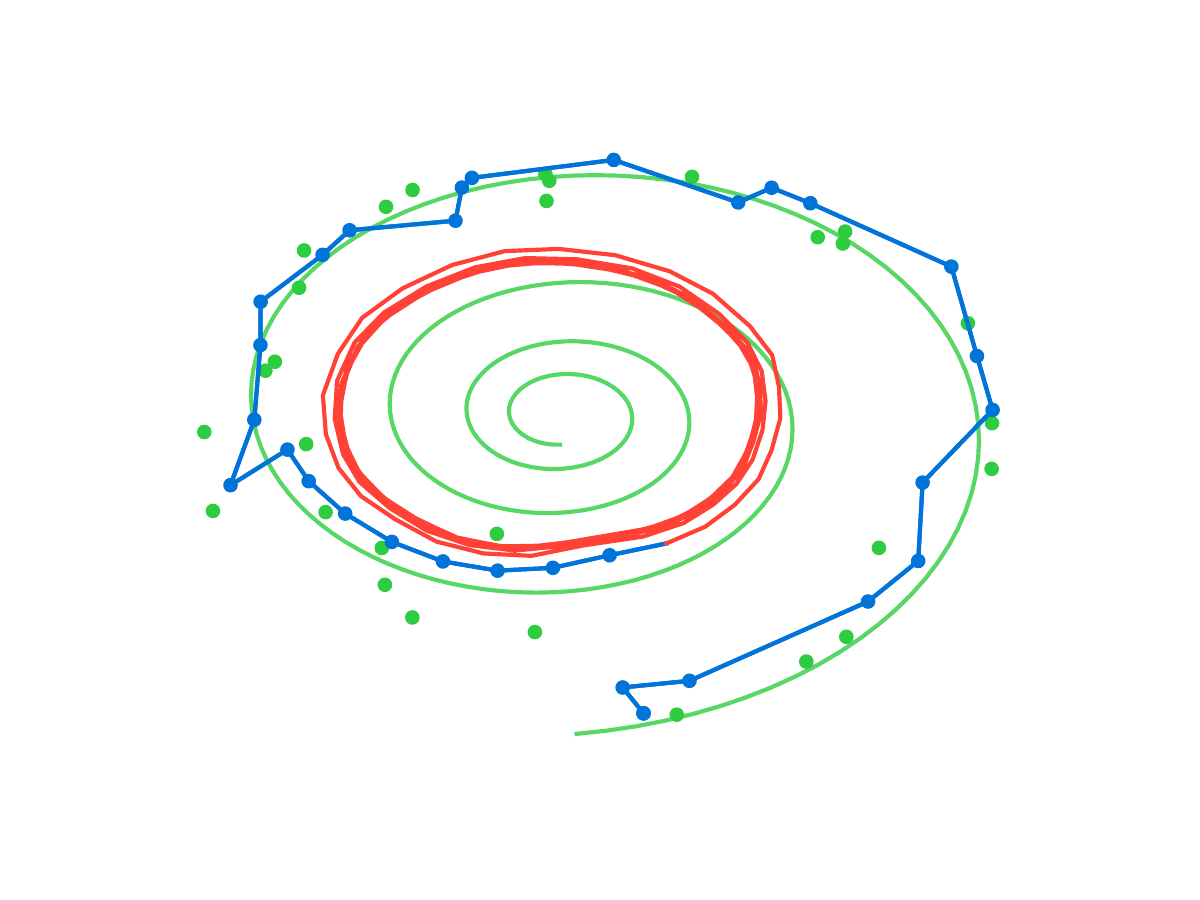}
		\caption{Recurrent Neural Network}
		\label{subfig:RNN}
	\end{subfigure}%
	\hfill
	\begin{subfigure}[b]{0.33\linewidth}
		\centering
		\includegraphics[trim={0.8cm 0.8cm 0.8cm 1cm},clip,width=0.49\linewidth]{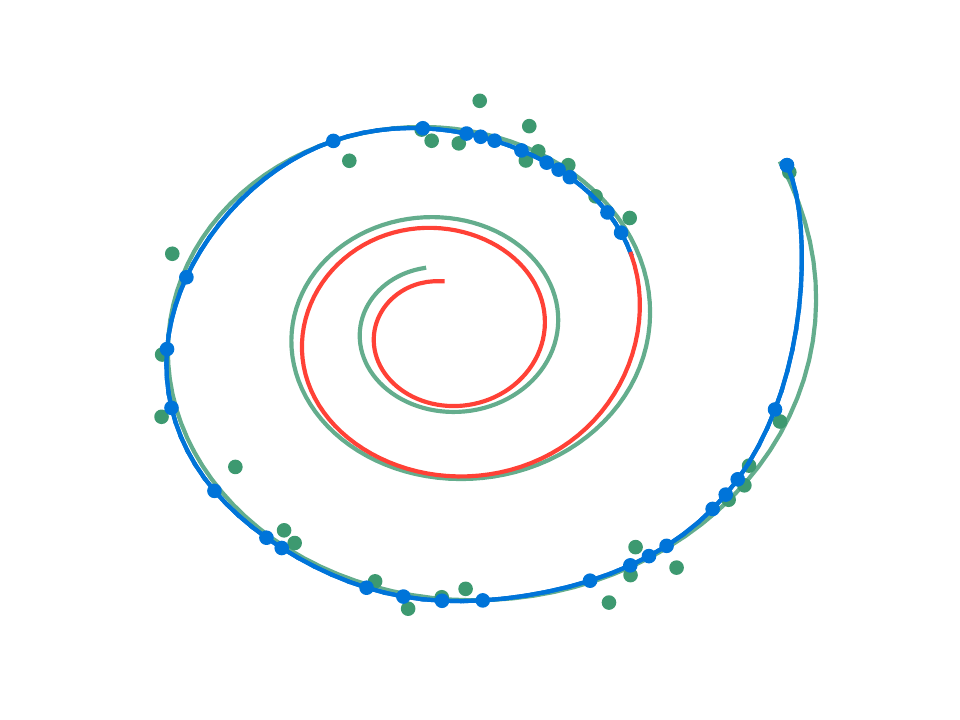}%
		\hfill
		\includegraphics[trim={0.8cm 0.8cm 0.8cm 1cm},clip,width=0.49\linewidth]{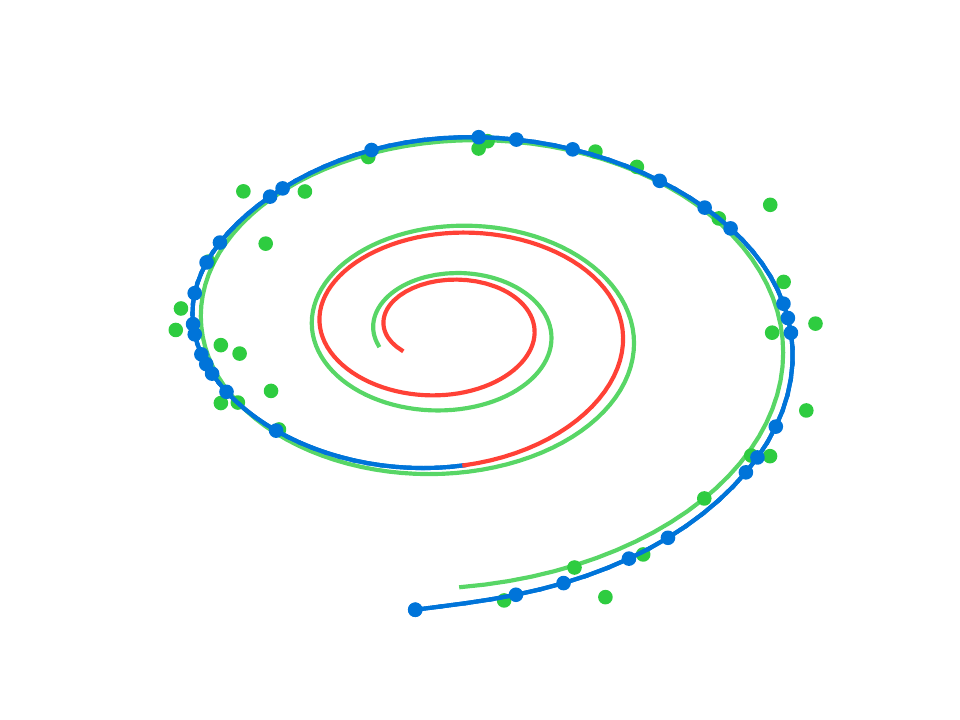}
		\caption{Latent ODE}
		\label{subfig:method}
	\end{subfigure}%
	\hfill
	\begin{subfigure}[b]{0.11\linewidth}
		\centering
		\includegraphics[trim={1.15cm 1.7cm 1.15cm 1.7cm}, clip, width=\linewidth]{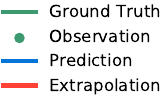}
		\vspace{6mm}
	\end{subfigure}%
	\hfill
	\begin{subfigure}[b]{0.20\linewidth}
		\centering
		\includegraphics[trim={0.7cm 1.4cm 0.7cm 0.7cm},clip,width=\linewidth]{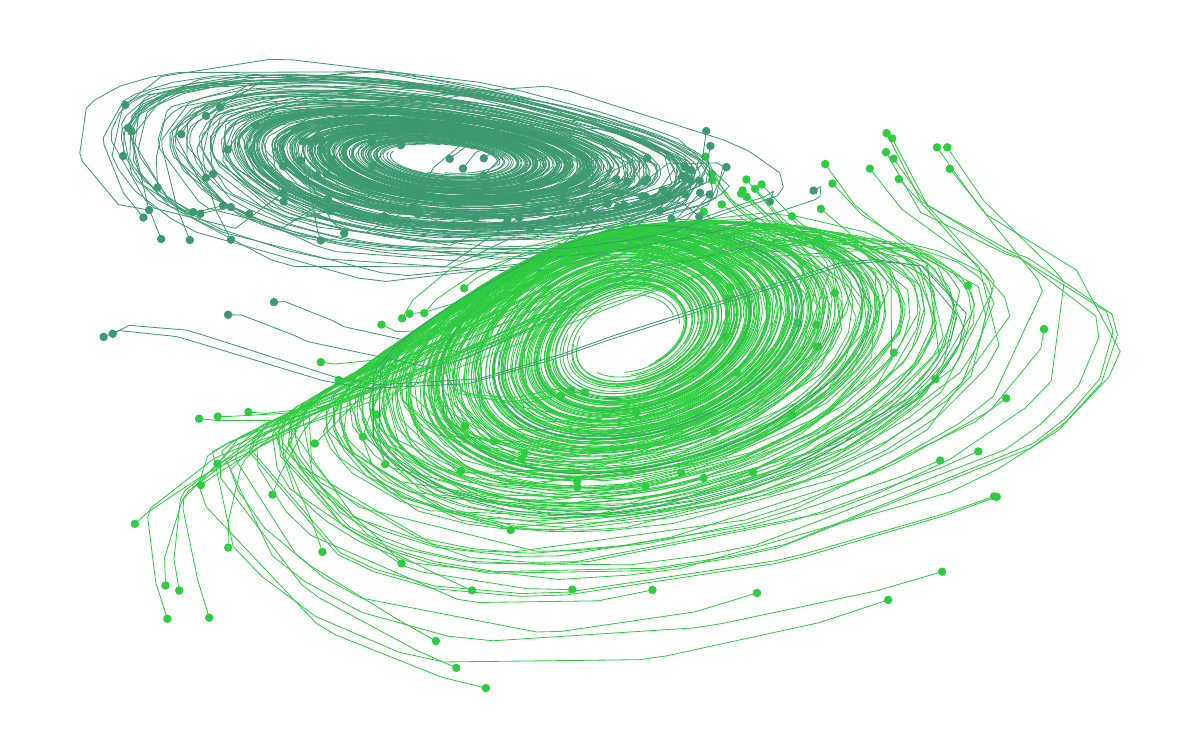}
		\caption{Latent Trajectories}
		\label{subfig:Latent-traj}
	\end{subfigure}
	\caption{(\subref{subfig:RNN}):
		Reconstruction and extrapolation of spirals with irregular time points by a recurrent neural network.
		(\subref{subfig:method}):
		Reconstructions and extrapolations by a latent neural ODE.
		Blue curve shows model prediction. %
		Red shows extrapolation. %
		(\subref{subfig:Latent-traj})
		A projection of inferred 4-dimensional latent ODE trajectories onto their first two dimensions.
		Color indicates the direction of the corresponding trajectory.
		The model has learned latent dynamics which distinguish the two directions.}
	\label{fig:ode_ctsm_examples}
\end{figure}%
We can parameterize $\lograte(\cdot)$ using another neural network.
Conveniently, we can evaluate both the latent trajectory and the Poisson process likelihood together in a single call to an ODE solver.
Figure~\ref{fig:poisson} shows the event rate learned by such a model on a toy dataset.

A Poisson process likelihood on observation times
can be combined with a data likelihood 
to jointly model all observations and the times at which they were made.

\subsection{Time-series Latent ODE Experiments}
We investigate the ability of the latent ODE model to fit and extrapolate time series.
The recognition network is an RNN with 25 hidden units.
We use a 4-dimensional latent space.
We parameterize the dynamics function $f$ with a one-hidden-layer network with 20 hidden units.
The decoder computing $p(\vx_{t_i}|\latent_{t_i})$ is another neural network with one hidden layer with 20 hidden units.
Our baseline was a recurrent neural net with 25 hidden units trained to minimize negative Gaussian log-likelihood.
We trained a second version of this RNN whose inputs were concatenated with the time difference to the next observation to aid RNN with irregular observations.%

\paragraph{Bi-directional spiral dataset}
We generated a dataset of 1000 2-dimensional spirals, each starting at a different point, sampled at 100 equally-spaced timesteps.
The dataset contains two types of spirals: half are clockwise while the other half counter-clockwise.
To make the task more realistic, we add gaussian noise to the observations.

\paragraph{Time series with irregular time points}

To generate irregular timestamps, we randomly sample points from each trajectory without replacement ($n=\{30, 50, 100\}$).
We report predictive root-mean-squared error (RMSE) on 100 time points extending beyond those that were used for training.
Table~\ref{tab:ode_ctsm_results_predictive} shows that the latent ODE has substantially lower predictive RMSE.

\begin{table}[h]
	\caption{Predictive RMSE on test set}
	\centering
	\begin{tabular}{lcccc}
		\# Observations & 30/100 & 50/100 & 100/100 \\
		\midrule
		\multicolumn{1}{l}{RNN} & 0.3937	& 0.3202	& 0.1813 \\
		\multicolumn{1}{l}{\method{}} & \textbf{0.1642} & \textbf{0.1502}	& \textbf{0.1346} \\
	\end{tabular}
	\label{tab:ode_ctsm_results_predictive}
\end{table}
Figure~\ref{fig:ode_ctsm_examples} shows examples of spiral reconstructions with 30 sub-sampled points.
Reconstructions from the latent ODE were obtained by sampling from the posterior over latent trajectories and decoding it to data-space.
Examples with varying number of time points are shown in Appendix~\ref{seq:extra}.
We observed that reconstructions and extrapolations are consistent with the ground truth regardless of number of observed points and despite the noise.

\paragraph{Latent space interpolation}
Figure~\ref{subfig:Latent-traj} shows latent trajectories projected onto the first two dimensions of the latent space.
The trajectories form two separate clusters of trajectories, one decoding to clockwise spirals, the other to counter-clockwise.
Figure~\ref{fig:varying_components_z0} shows that the latent trajectories change smoothly as a function of the initial point $\sol(t_0)$, switching from a clockwise to a counter-clockwise spiral.

\begin{figure}
	\centering
	\includegraphics[width=\linewidth]{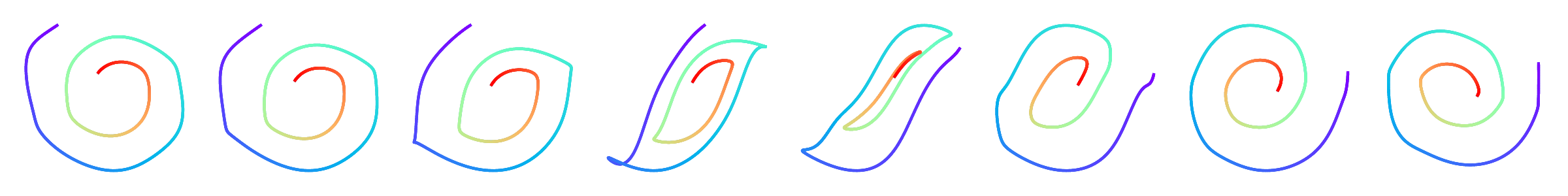}
	\caption{Data-space trajectories decoded from varying one dimension of $\sol_{t_0}$. Color indicates progression through time, starting at purple and ending at red.
		Note that the trajectories on the left are counter-clockwise, while the trajectories on the right are clockwise.}
	\label{fig:varying_components_z0}
\end{figure}

\section{Scope and Limitations}
\label{sec:scope}

\paragraph{Minibatching}
The use of mini-batches is less straightforward than for standard neural networks.
One can still batch together evaluations through the ODE solver by concatenating the states of each batch element together, creating a combined ODE with dimension $D \times K$.
In some cases, controlling error on all batch elements together might require evaluating the combined system $K$ times more often than if each system was solved individually.
However, in practice the number of evaluations did not increase substantially when using minibatches.

\paragraph{Uniqueness}
When do continuous dynamics have a unique solution?
Picard's existence theorem~\citep{coddington1955theory} states that the solution to an initial value problem exists and is unique if the differential equation is uniformly Lipschitz continuous in $\sol$ and continuous in $t$.
This theorem holds for our model if the neural network has finite weights and uses Lipschitz nonlinearities, such as \texttt{tanh} or \texttt{relu}.

\paragraph{Setting tolerances}
Our framework allows the user to trade off speed for precision, but requires the user to choose an error tolerance on both the forward and reverse passes during training.
For sequence modeling, the default value of \texttt{1.5e-8} was used. In the classification and density estimation experiments, we were able to reduce the tolerance to \texttt{1e-3} and \texttt{1e-5}, respectively, without degrading performance.

\paragraph{Reconstructing forward trajectories}
Reconstructing the state trajectory by running the dynamics backwards can introduce extra numerical error if the reconstructed trajectory diverges from the original.
This problem can be addressed by checkpointing: storing intermediate values of $\sol$ on the forward pass, and reconstructing the exact forward trajectory by re-integrating from those points.
We did not find this to be a practical problem, and we informally checked that reversing many layers of continuous normalizing flows with default tolerances recovered the initial states.

\section{Related Work}
The use of the adjoint method for training continuous-time neural networks was previously proposed~\citep{lecun1988theoretical, pearlmutter1995gradient}, though was not demonstrated practically.
The interpretation of residual networks~\citep{he2016deep} as approximate ODE solvers spurred research into exploiting reversibility and approximate computation in ResNets~\citep{chang2017reversible, lu2017beyond}.
We demonstrate these same properties in more generality by directly using an ODE solver.

\paragraph{Adaptive computation}
One can adapt computation time by training secondary neural networks to choose the number of evaluations of recurrent or residual networks~\citep{graves2016adaptive, jernite2016variable, figurnov2017spatially,chang2018multilevel}.
However, this introduces overhead both at training and test time, and extra parameters that need to be fit.
In contrast, ODE solvers offer well-studied, computationally cheap, and generalizable rules for adapting the amount of computation.

\paragraph{Constant memory backprop through reversibility}
Recent work developed reversible versions of residual networks~\citep{gomez2017reversible, haber2017stable, chang2017reversible}, which gives the same constant memory advantage as our approach.
However, these methods require restricted architectures, which partition the hidden units. %
Our approach does not have these restrictions.

\paragraph{Learning differential equations}
Much recent work has proposed learning differential equations from data.
One can train feed-forward or recurrent neural networks to approximate a differential equation~\citep{2018Raissi,raissi2018multistep,2017PDEnetfromData}, with applications such as fluid simulation~\citep{wiewel2018latent}.
There is also significant work on connecting Gaussian Processes (GPs) and ODE solvers~\citep{SchoDuvHen2014ode}.
GPs have been adapted to fit differential equations~\citep{raissi2018numerical} and can naturally model continuous-time effects and interventions~\citep{soleimani2017treatment,schulam2017if}.
\citet{2018blackboxSVIforSDE} use stochastic variational inference to recover the solution of a given stochastic differential equation.

\paragraph{Differentiating through ODE solvers}
The \texttt{dolfin} library~\citep{dolfin2013} implements adjoint computation for general ODE and PDE solutions, but only by backpropagating through the individual operations of the forward solver.
The Stan library~\citep{carpenter2015stan} implements gradient estimation through ODE solutions using forward sensitivity analysis.
However, forward sensitivity analysis is quadratic-time in the number of variables, whereas the adjoint sensitivity analysis is linear~\citep{carpenter2015stan,zhang2014fatode}.
\citet{melicher2017fast} used the adjoint method to train bespoke latent dynamic models.

In contrast, by providing a generic vector-Jacobian product, we allow an ODE solver to be trained end-to-end with any other differentiable model components. While use of vector-Jacobian products for solving the adjoint method has been explored in optimal control~\citep{andersson2013general,Andersson2018}, we highlight the potential of a general integration of black-box ODE solvers into automatic differentiation~\citep{baydin2017automatic} for deep learning and generative modeling.

\section{Conclusion}

We investigated the use of black-box ODE solvers as a model component, developing new models for time-series modeling, supervised learning, and density estimation.
These models are evaluated adaptively, and allow explicit control of the tradeoff between computation speed and accuracy.
Finally, we derived an instantaneous version of the change of variables formula, and developed continuous-time normalizing flows, which can scale to large layer sizes.

\section{Acknowledgements}
We thank Wenyi Wang and Geoff Roeder for help with proofs, and Daniel Duckworth, Ethan Fetaya, Hossein Soleimani, Eldad Haber, Ken Caluwaerts, Daniel Flam-Shepherd, and Harry Braviner for feedback. We thank Chris Rackauckas, Dougal Maclaurin, and Matthew James Johnson for helpful discussions.
We also thank Yuval Frommer for pointing out an unsupported claim about parameter efficiency.
We thank Leonardo Vanni, Giuseppe Bassi, Kabir Ben Nefla, Gaia Grossi, Anna Notaro, and Beatrice Porta for finding and diagnosing an error in the function evaluation counts reported in Figure~\ref{fig:odenet-plots}c.

\bibliography{bibfile}

\begin{thebibliography}{61}
\providecommand{\natexlab}[1]{#1}
\providecommand{\url}[1]{\texttt{#1}}
\expandafter\ifx\csname urlstyle\endcsname\relax
  \providecommand{\doi}[1]{doi: #1}\else
  \providecommand{\doi}{doi: \begingroup \urlstyle{rm}\Url}\fi

\bibitem[{\'A}lvarez and Lawrence(2011)]{alvarez2011computationally}
Mauricio~A {\'A}lvarez and Neil~D Lawrence.
\newblock Computationally efficient convolved multiple output {G}aussian
  processes.
\newblock \emph{Journal of Machine Learning Research}, 12\penalty0
  (May):\penalty0 1459--1500, 2011.

\bibitem[Amos and Kolter(2017)]{amos2017optnet}
Brandon Amos and J~Zico Kolter.
\newblock {OptNet}: Differentiable optimization as a layer in neural networks.
\newblock In \emph{International Conference on Machine Learning}, pages
  136--145, 2017.

\bibitem[Andersson(2013)]{andersson2013general}
Joel Andersson.
\newblock \emph{A general-purpose software framework for dynamic optimization}.
\newblock PhD thesis, 2013.

\bibitem[Andersson et~al.(In Press, 2018)Andersson, Gillis, Horn, Rawlings, and
  Diehl]{Andersson2018}
Joel A~E Andersson, Joris Gillis, Greg Horn, James~B Rawlings, and Moritz
  Diehl.
\newblock {CasADi} -- {A} software framework for nonlinear optimization and
  optimal control.
\newblock \emph{Mathematical Programming Computation}, In Press, 2018.

\bibitem[Baydin et~al.(2018)Baydin, Pearlmutter, Radul, and
  Siskind]{baydin2017automatic}
Atilim~Gunes Baydin, Barak~A Pearlmutter, Alexey~Andreyevich Radul, and
  Jeffrey~Mark Siskind.
\newblock Automatic differentiation in machine learning: a survey.
\newblock \emph{Journal of machine learning research}, 18\penalty0
  (153):\penalty0 1--153, 2018.

\bibitem[Berg et~al.(2018)Berg, Hasenclever, Tomczak, and
  Welling]{berg2018sylvester}
Rianne van~den Berg, Leonard Hasenclever, Jakub~M Tomczak, and Max Welling.
\newblock Sylvester normalizing flows for variational inference.
\newblock \emph{arXiv preprint arXiv:1803.05649}, 2018.

\bibitem[Carpenter et~al.(2015)Carpenter, Hoffman, Brubaker, Lee, Li, and
  Betancourt]{carpenter2015stan}
Bob Carpenter, Matthew~D Hoffman, Marcus Brubaker, Daniel Lee, Peter Li, and
  Michael Betancourt.
\newblock The {Stan} math library: Reverse-mode automatic differentiation in
  c++.
\newblock \emph{arXiv preprint arXiv:1509.07164}, 2015.

\bibitem[Chang et~al.(2017)Chang, Meng, Haber, Ruthotto, Begert, and
  Holtham]{chang2017reversible}
Bo~Chang, Lili Meng, Eldad Haber, Lars Ruthotto, David Begert, and Elliot
  Holtham.
\newblock Reversible architectures for arbitrarily deep residual neural
  networks.
\newblock \emph{arXiv preprint arXiv:1709.03698}, 2017.

\bibitem[Chang et~al.(2018)Chang, Meng, Haber, Tung, and
  Begert]{chang2018multilevel}
Bo~Chang, Lili Meng, Eldad Haber, Frederick Tung, and David Begert.
\newblock Multi-level residual networks from dynamical systems view.
\newblock In \emph{International Conference on Learning Representations}, 2018.
\newblock URL \url{https://openreview.net/forum?id=SyJS-OgR-}.

\bibitem[Che et~al.(2018)Che, Purushotham, Cho, Sontag, and
  Liu]{che_purushotham_cho_sontag_liu_2016}
Zhengping Che, Sanjay Purushotham, Kyunghyun Cho, David Sontag, and Yan Liu.
\newblock Recurrent neural networks for multivariate time series with missing
  values.
\newblock \emph{Scientific Reports}, 8\penalty0 (1):\penalty0 6085, 2018.
\newblock URL \url{https://doi.org/10.1038/s41598-018-24271-9}.

\bibitem[Choi et~al.(2016)Choi, Bahadori, Schuetz, Stewart, and
  Sun]{pmlr-v56-Choi16}
Edward Choi, Mohammad~Taha Bahadori, Andy Schuetz, Walter~F. Stewart, and
  Jimeng Sun.
\newblock Doctor {AI}: Predicting clinical events via recurrent neural
  networks.
\newblock In \emph{Proceedings of the 1st Machine Learning for Healthcare
  Conference}, volume~56 of \emph{Proceedings of Machine Learning Research},
  pages 301--318. PMLR, 18--19 Aug 2016.
\newblock URL \url{http://proceedings.mlr.press/v56/Choi16.html}.

\bibitem[Coddington and Levinson(1955)]{coddington1955theory}
Earl~A Coddington and Norman Levinson.
\newblock \emph{Theory of ordinary differential equations}.
\newblock Tata McGraw-Hill Education, 1955.

\bibitem[Dinh et~al.(2014)Dinh, Krueger, and Bengio]{dinh2014nice}
Laurent Dinh, David Krueger, and Yoshua Bengio.
\newblock {NICE}: Non-linear independent components estimation.
\newblock \emph{arXiv preprint arXiv:1410.8516}, 2014.

\bibitem[Du et~al.(2016)Du, Dai, Trivedi, Upadhyay, Gomez-Rodriguez, and
  Song]{du2016recurrent}
Nan Du, Hanjun Dai, Rakshit Trivedi, Utkarsh Upadhyay, Manuel Gomez-Rodriguez,
  and Le~Song.
\newblock Recurrent marked temporal point processes: Embedding event history to
  vector.
\newblock In \emph{International Conference on Knowledge Discovery and Data
  Mining}, pages 1555--1564. ACM, 2016.

\bibitem[Farrell et~al.(2013)Farrell, Ham, Funke, and Rognes]{dolfin2013}
Patrick Farrell, David Ham, Simon Funke, and Marie Rognes.
\newblock Automated derivation of the adjoint of high-level transient finite
  element programs.
\newblock \emph{SIAM Journal on Scientific Computing}, 2013.

\bibitem[Figurnov et~al.(2017)Figurnov, Collins, Zhu, Zhang, Huang, Vetrov, and
  Salakhutdinov]{figurnov2017spatially}
Michael Figurnov, Maxwell~D Collins, Yukun Zhu, Li~Zhang, Jonathan Huang,
  Dmitry Vetrov, and Ruslan Salakhutdinov.
\newblock Spatially adaptive computation time for residual networks.
\newblock \emph{arXiv preprint}, 2017.

\bibitem[{Futoma} et~al.(2017){Futoma}, {Hariharan}, and
  {Heller}]{futoma_hariharan_heller_2017}
J.~{Futoma}, S.~{Hariharan}, and K.~{Heller}.
\newblock {Learning to Detect Sepsis with a Multitask {G}aussian Process RNN
  Classifier}.
\newblock \emph{ArXiv e-prints}, 2017.

\bibitem[Gomez et~al.(2017)Gomez, Ren, Urtasun, and
  Grosse]{gomez2017reversible}
Aidan~N Gomez, Mengye Ren, Raquel Urtasun, and Roger~B Grosse.
\newblock The reversible residual network: Backpropagation without storing
  activations.
\newblock In \emph{Advances in Neural Information Processing Systems}, pages
  2211--2221, 2017.

\bibitem[Graves(2016)]{graves2016adaptive}
Alex Graves.
\newblock Adaptive computation time for recurrent neural networks.
\newblock \emph{arXiv preprint arXiv:1603.08983}, 2016.

\bibitem[Ha et~al.(2016)Ha, Dai, and Le]{ha2016hypernetworks}
David Ha, Andrew Dai, and Quoc~V Le.
\newblock Hypernetworks.
\newblock \emph{arXiv preprint arXiv:1609.09106}, 2016.

\bibitem[Haber and Ruthotto(2017)]{haber2017stable}
Eldad Haber and Lars Ruthotto.
\newblock Stable architectures for deep neural networks.
\newblock \emph{Inverse Problems}, 34\penalty0 (1):\penalty0 014004, 2017.

\bibitem[Hairer et~al.(1987)Hairer, N{\o}rsett, and
  Wanner]{hairer87:_solvin_ordin_differ_equat_i}
E.~Hairer, S.P. N{\o}rsett, and G.~Wanner.
\newblock \emph{{Solving Ordinary Differential Equations {I} -- Nonstiff
  Problems}}.
\newblock Springer, 1987.

\bibitem[He et~al.(2016{\natexlab{a}})He, Zhang, Ren, and Sun]{he2016deep}
Kaiming He, Xiangyu Zhang, Shaoqing Ren, and Jian Sun.
\newblock Deep residual learning for image recognition.
\newblock In \emph{Proceedings of the IEEE conference on computer vision and
  pattern recognition}, pages 770--778, 2016{\natexlab{a}}.

\bibitem[He et~al.(2016{\natexlab{b}})He, Zhang, Ren, and Sun]{he2016identity}
Kaiming He, Xiangyu Zhang, Shaoqing Ren, and Jian Sun.
\newblock Identity mappings in deep residual networks.
\newblock In \emph{European conference on computer vision}, pages 630--645.
  Springer, 2016{\natexlab{b}}.

\bibitem[Hinton et~al.(2012)Hinton, Srivastava, and Swersky]{hinton2012neural}
Geoffrey Hinton, Nitish Srivastava, and Kevin Swersky.
\newblock Neural networks for machine learning lecture 6a overview of
  mini-batch gradient descent, 2012.

\bibitem[Jernite et~al.(2016)Jernite, Grave, Joulin, and
  Mikolov]{jernite2016variable}
Yacine Jernite, Edouard Grave, Armand Joulin, and Tomas Mikolov.
\newblock Variable computation in recurrent neural networks.
\newblock \emph{arXiv preprint arXiv:1611.06188}, 2016.

\bibitem[Kingma and Ba(2014)]{kingma2014adam}
Diederik~P Kingma and Jimmy Ba.
\newblock {Adam}: A method for stochastic optimization.
\newblock \emph{arXiv preprint arXiv:1412.6980}, 2014.

\bibitem[Kingma and Welling(2014)]{kingma2013autoencoding}
Diederik~P. Kingma and Max Welling.
\newblock Auto-encoding variational {B}ayes.
\newblock \emph{International Conference on Learning Representations}, 2014.

\bibitem[Kingma et~al.(2016)Kingma, Salimans, Jozefowicz, Chen, Sutskever, and
  Welling]{kingma2016improved}
Diederik~P Kingma, Tim Salimans, Rafal Jozefowicz, Xi~Chen, Ilya Sutskever, and
  Max Welling.
\newblock Improved variational inference with inverse autoregressive flow.
\newblock In \emph{Advances in Neural Information Processing Systems}, pages
  4743--4751, 2016.

\bibitem[Kutta(1901)]{Kutta}
W.~Kutta.
\newblock {Beitrag zur n{\"a}herungsweisen {I}ntegration totaler
  {D}ifferentialgleichungen}.
\newblock \emph{Zeitschrift f{{\"u}r} Mathematik und Physik}, 46:\penalty0
  435--453, 1901.

\bibitem[LeCun et~al.(1988)LeCun, Touresky, Hinton, and
  Sejnowski]{lecun1988theoretical}
Yann LeCun, D~Touresky, G~Hinton, and T~Sejnowski.
\newblock A theoretical framework for back-propagation.
\newblock In \emph{Proceedings of the 1988 connectionist models summer school},
  volume~1, pages 21--28. CMU, Pittsburgh, Pa: Morgan Kaufmann, 1988.

\bibitem[LeCun et~al.(1998)LeCun, Bottou, Bengio, and
  Haffner]{lecun1998gradient}
Yann LeCun, L{\'e}on Bottou, Yoshua Bengio, and Patrick Haffner.
\newblock Gradient-based learning applied to document recognition.
\newblock \emph{Proceedings of the IEEE}, 86\penalty0 (11):\penalty0
  2278--2324, 1998.

\bibitem[Li(2017)]{li2017time}
Yang Li.
\newblock Time-dependent representation for neural event sequence prediction.
\newblock \emph{arXiv preprint arXiv:1708.00065}, 2017.

\bibitem[Lipton et~al.(2016)Lipton, Kale, and Wetzel]{pmlr-v56-Lipton16}
Zachary~C Lipton, David Kale, and Randall Wetzel.
\newblock Directly modeling missing data in sequences with {RNN}s: Improved
  classification of clinical time series.
\newblock In \emph{Proceedings of the 1st Machine Learning for Healthcare
  Conference}, volume~56 of \emph{Proceedings of Machine Learning Research},
  pages 253--270. PMLR, 18--19 Aug 2016.
\newblock URL \url{http://proceedings.mlr.press/v56/Lipton16.html}.

\bibitem[{Long} et~al.(2017){Long}, {Lu}, {Ma}, and {Dong}]{2017PDEnetfromData}
Z.~{Long}, Y.~{Lu}, X.~{Ma}, and B.~{Dong}.
\newblock {{PDE-Net}: Learning {PDE}s from Data}.
\newblock \emph{ArXiv e-prints}, 2017.

\bibitem[Lu et~al.(2017)Lu, Zhong, Li, and Dong]{lu2017beyond}
Yiping Lu, Aoxiao Zhong, Quanzheng Li, and Bin Dong.
\newblock Beyond finite layer neural networks: Bridging deep architectures and
  numerical differential equations.
\newblock \emph{arXiv preprint arXiv:1710.10121}, 2017.

\bibitem[Maclaurin et~al.(2015)Maclaurin, Duvenaud, and
  Adams]{maclaurin2015autograd}
Dougal Maclaurin, David Duvenaud, and Ryan~P Adams.
\newblock {Autograd}: Reverse-mode differentiation of native {P}ython.
\newblock In \emph{ICML workshop on Automatic Machine Learning}, 2015.

\bibitem[Mei and Eisner(2017)]{mei2017neural}
Hongyuan Mei and Jason~M Eisner.
\newblock The neural {H}awkes process: A neurally self-modulating multivariate
  point process.
\newblock In \emph{Advances in Neural Information Processing Systems}, pages
  6757--6767, 2017.

\bibitem[Melicher et~al.(2017)Melicher, Haber, and Vanroose]{melicher2017fast}
Valdemar Melicher, Tom Haber, and Wim Vanroose.
\newblock Fast derivatives of likelihood functionals for {ODE} based models
  using adjoint-state method.
\newblock \emph{Computational Statistics}, 32\penalty0 (4):\penalty0
  1621--1643, 2017.

\bibitem[Palm(1943)]{palm1943intensitatsschwankungen}
Conny Palm.
\newblock Intensit{\"a}tsschwankungen im fernsprechverker.
\newblock \emph{Ericsson Technics}, 1943.

\bibitem[Paszke et~al.(2017)Paszke, Gross, Chintala, Chanan, Yang, DeVito, Lin,
  Desmaison, Antiga, and Lerer]{paszke2017automatic}
Adam Paszke, Sam Gross, Soumith Chintala, Gregory Chanan, Edward Yang, Zachary
  DeVito, Zeming Lin, Alban Desmaison, Luca Antiga, and Adam Lerer.
\newblock Automatic differentiation in pytorch.
\newblock 2017.

\bibitem[Pearlmutter(1995)]{pearlmutter1995gradient}
Barak~A Pearlmutter.
\newblock Gradient calculations for dynamic recurrent neural networks: A
  survey.
\newblock \emph{IEEE Transactions on Neural networks}, 6\penalty0 (5):\penalty0
  1212--1228, 1995.

\bibitem[Pontryagin et~al.(1962)Pontryagin, Mishchenko, Boltyanskii, and
  Gamkrelidze]{pontryagin1962mathematical}
Lev~Semenovich Pontryagin, EF~Mishchenko, VG~Boltyanskii, and RV~Gamkrelidze.
\newblock The mathematical theory of optimal processes.
\newblock 1962.

\bibitem[{Raissi} and {Karniadakis}(2018)]{2018Raissi}
M.~{Raissi} and G.~E. {Karniadakis}.
\newblock {Hidden physics models: Machine learning of nonlinear partial
  differential equations}.
\newblock \emph{Journal of Computational Physics}, pages 125--141, 2018.

\bibitem[Raissi et~al.(2018{\natexlab{a}})Raissi, Perdikaris, and
  Karniadakis]{raissi2018multistep}
Maziar Raissi, Paris Perdikaris, and George~Em Karniadakis.
\newblock Multistep neural networks for data-driven discovery of nonlinear
  dynamical systems.
\newblock \emph{arXiv preprint arXiv:1801.01236}, 2018{\natexlab{a}}.

\bibitem[Raissi et~al.(2018{\natexlab{b}})Raissi, Perdikaris, and
  Karniadakis]{raissi2018numerical}
Maziar Raissi, Paris Perdikaris, and George~Em Karniadakis.
\newblock Numerical {G}aussian processes for time-dependent and nonlinear
  partial differential equations.
\newblock \emph{SIAM Journal on Scientific Computing}, 40\penalty0
  (1):\penalty0 A172--A198, 2018{\natexlab{b}}.

\bibitem[Rezende et~al.(2014)Rezende, Mohamed, and
  Wierstra]{rezende2014stochastic}
Danilo~J Rezende, Shakir Mohamed, and Daan Wierstra.
\newblock Stochastic backpropagation and approximate inference in deep
  generative models.
\newblock In \emph{Proceedings of the 31st International Conference on Machine
  Learning}, pages 1278--1286, 2014.

\bibitem[Rezende and Mohamed(2015)]{rezende2015variational}
Danilo~Jimenez Rezende and Shakir Mohamed.
\newblock Variational inference with normalizing flows.
\newblock \emph{arXiv preprint arXiv:1505.05770}, 2015.

\bibitem[Runge(1895)]{Runge}
C.~Runge.
\newblock {{\"U}ber die numerische {A}ufl{\"o}sung von
  {D}ifferentialgleichungen}.
\newblock \emph{Mathematische Annalen}, 46:\penalty0 167--178, 1895.

\bibitem[Ruthotto and Haber(2018)]{ruthotto2018deep}
Lars Ruthotto and Eldad Haber.
\newblock Deep neural networks motivated by partial differential equations.
\newblock \emph{arXiv preprint arXiv:1804.04272}, 2018.

\bibitem[Ryder et~al.(2018)Ryder, Golightly, McGough, and
  Prangle]{2018blackboxSVIforSDE}
T.~Ryder, A.~Golightly, A.~S. McGough, and D.~Prangle.
\newblock {Black-box Variational Inference for Stochastic Differential
  Equations}.
\newblock \emph{ArXiv e-prints}, 2018.

\bibitem[Schober et~al.(2014)Schober, Duvenaud, and Hennig]{SchoDuvHen2014ode}
Michael Schober, David Duvenaud, and Philipp Hennig.
\newblock Probabilistic {ODE} solvers with {R}unge-{K}utta means.
\newblock In \emph{Advances in Neural Information Processing Systems 25}, 2014.

\bibitem[Schulam and Saria(2017)]{schulam2017if}
Peter Schulam and Suchi Saria.
\newblock What-if reasoning with counterfactual {G}aussian processes.
\newblock \emph{arXiv preprint arXiv:1703.10651}, 2017.

\bibitem[Soleimani et~al.(2017{\natexlab{a}})Soleimani, Hensman, and
  Saria]{soleimani2017scalable}
Hossein Soleimani, James Hensman, and Suchi Saria.
\newblock Scalable joint models for reliable uncertainty-aware event
  prediction.
\newblock \emph{IEEE transactions on pattern analysis and machine
  intelligence}, 2017{\natexlab{a}}.

\bibitem[Soleimani et~al.(2017{\natexlab{b}})Soleimani, Subbaswamy, and
  Saria]{soleimani2017treatment}
Hossein Soleimani, Adarsh Subbaswamy, and Suchi Saria.
\newblock Treatment-response models for counterfactual reasoning with
  continuous-time, continuous-valued interventions.
\newblock \emph{arXiv preprint arXiv:1704.02038}, 2017{\natexlab{b}}.

\bibitem[Stam(1999)]{stam1999stable}
Jos Stam.
\newblock Stable fluids.
\newblock In \emph{Proceedings of the 26th annual conference on Computer
  graphics and interactive techniques}, pages 121--128. ACM
  Press/Addison-Wesley Publishing Co., 1999.

\bibitem[Stapor et~al.(2018)Stapor, Froehlich, and
  Hasenauer]{stapor2018optimization}
Paul Stapor, Fabian Froehlich, and Jan Hasenauer.
\newblock Optimization and uncertainty analysis of {ODE} models using second
  order adjoint sensitivity analysis.
\newblock \emph{bioRxiv}, page 272005, 2018.

\bibitem[Tomczak and Welling(2016)]{tomczak2016improving}
Jakub~M Tomczak and Max Welling.
\newblock Improving variational auto-encoders using {H}ouseholder flow.
\newblock \emph{arXiv preprint arXiv:1611.09630}, 2016.

\bibitem[Vanni et~al.(2026)Vanni, Bassi, Ben~Nefla, Grossi, Notaro, and
  Porta]{vanni2026replication}
Leonardo Vanni, Giuseppe Bassi, Kabir Ben~Nefla, Gaia Grossi, Anna Notaro, and
  Beatrice Porta.
\newblock {[Re]} {A}ugmented neural {ODE}s.
\newblock Submitted to \emph{ReScience C}, 2026.

\bibitem[Wiewel et~al.(2018)Wiewel, Becher, and Thuerey]{wiewel2018latent}
Steffen Wiewel, Moritz Becher, and Nils Thuerey.
\newblock Latent-space physics: Towards learning the temporal evolution of
  fluid flow.
\newblock \emph{arXiv preprint arXiv:1802.10123}, 2018.

\bibitem[Zhang and Sandu(2014)]{zhang2014fatode}
Hong Zhang and Adrian Sandu.
\newblock Fatode: a library for forward, adjoint, and tangent linear
  integration of {ODE}s.
\newblock \emph{SIAM Journal on Scientific Computing}, 36\penalty0
  (5):\penalty0 C504--C523, 2014.

\end{thebibliography}
\small {\bibliographystyle{plainnat}}

\clearpage

\begin{appendices}

\section{Proof of the Instantaneous Change of Variables Theorem}\label{sec:proof}
\begin{theorem*}[Instantaneous Change of Variables]
Let $\cnfx(t)$ be a finite continuous random variable with probability $p(\cnfx(t))$ dependent on time.
Let $\frac{d\cnfx}{dt} = f(\cnfx(t), t)$ be a differential equation describing a continuous-in-time transformation of $\cnfx(t)$.
Assuming that $f$ is uniformly Lipschitz continuous in $\cnfx$ and continuous in $t$, then the change in log probability also follows a differential equation:
\begin{align*}
\frac{\partial \log p(\cnfx(t))}{\partial t} = -\tr\left(\frac{d f}{d \cnfx}(t)\right)
\end{align*}
\end{theorem*}
\begin{proof}
To prove this theorem, we take the infinitesimal limit of finite changes of $\log p({\cnfx}(t))$ through time.
First we denote the transformation of ${\cnfx}$ over an $\epsilon$ change in time as%
\begin{align}
{\cnfx}(t+\epsilon) = T_\epsilon({\cnfx}(t))%
\end{align}
We assume that $f$ is Lipschitz continuous in ${\cnfx}(t)$ and continuous in $t$, so every initial value problem has a unique solution by Picard's existence theorem.
We also assume ${\cnfx}(t)$ is bounded.
These conditions imply that $f$, $T_\epsilon$, and $\frac{\partial}{\partial {\cnfx}}T_\epsilon$ are all bounded.
In the following, we use these conditions to exchange limits and products.

We can write the differential equation $\frac{\partial\log p({\cnfx}(t))}{\partial t}$ using the discrete change of variables formula, and the definition of the derivative:
\begin{align}
\frac{\partial\log p({\cnfx}(t))}{\partial t} &= \lim_{\epsilon\rightarrow0^+} \frac{\log p({\cnfx}(t)) - \log \left|\det \frac{\partial}{\partial {\cnfx}} T_\epsilon({\cnfx}(t)) \right| - \log p({\cnfx}(t))}{\epsilon} \\
&= -\lim_{\epsilon\rightarrow0^+} \frac{\log \left|\det \frac{\partial}{\partial {\cnfx}} T_\epsilon({\cnfx}(t)) \right| }{\epsilon} \\
&= -\lim_{\epsilon\rightarrow0^+} \frac{\frac{\partial}{\partial \epsilon}\log \left|\det \frac{\partial}{\partial {\cnfx}} T_\epsilon({\cnfx}(t)) \right| }{\frac{\partial}{\partial \epsilon}\epsilon}
\qquad \qquad \textnormal{(by L'H\^opital's rule)}\\
&= -\lim_{\epsilon\rightarrow0^+} \frac{\frac{\partial}{\partial \epsilon}\left|\det \frac{\partial}{\partial {\cnfx}} T_\epsilon({\cnfx}(t)) \right| }{\left|\det \frac{\partial}{\partial {\cnfx}} T_\epsilon({\cnfx}(t)) \right|}
\qquad \qquad \quad \left( \left. \frac{\partial \log({\cnfx})}{\partial {\cnfx}}\right\vert_{\cnfx = 1} = 1 \right)\\
&= -\underbrace{\left( \lim_{\epsilon\rightarrow0^+} \frac{\partial}{\partial \epsilon}\left|\det \frac{\partial}{\partial {\cnfx}} T_\epsilon({\cnfx}(t)) \right| \right)}_{\text{bounded}} \underbrace{\left( \lim_{\epsilon\rightarrow0^+} \frac{1 }{\left|\det \frac{\partial}{\partial {\cnfx}} T_\epsilon({\cnfx}(t)) \right|}  \right)}_{=1} \\
&= -\lim_{\epsilon\rightarrow0^+} \frac{\partial}{\partial \epsilon}\left|\det \frac{\partial}{\partial {\cnfx}} T_\epsilon({\cnfx}(t)) \right|
\end{align}
The derivative of the determinant can be expressed using Jacobi's formula, which gives
\begin{align}
    \frac{\partial\log p({\cnfx}(t))}{\partial t} &= -\lim_{\epsilon\rightarrow0^+} \tr \left( \text{adj}\left(\frac{\partial}{\partial {\cnfx}}T_\epsilon({\cnfx}(t))\right) \frac{\partial }{\partial \epsilon
    } \frac{\partial }{\partial {\cnfx}
    } T_\epsilon({\cnfx}(t)) \right) \\
    &= -\tr \left( \underbrace{\left(\lim_{\epsilon\rightarrow0^+} \text{adj}\left(\frac{\partial}{\partial {\cnfx}}T_\epsilon({\cnfx}(t))\right)\right)}_{=I} \left( \lim_{\epsilon\rightarrow0^+} \frac{\partial }{\partial \epsilon
    } \frac{\partial }{\partial {\cnfx}
    } T_\epsilon({\cnfx}(t))\right) \right) \\
    &= -\tr \left( \lim_{\epsilon\rightarrow0^+} \frac{\partial }{\partial \epsilon
    } \frac{\partial }{\partial {\cnfx}
    } T_\epsilon({\cnfx}(t))\right)
\end{align}
Substituting $T_\epsilon$ with its Taylor series expansion and taking the limit, we complete the proof.
\begin{align}
    \frac{\partial\log p({\cnfx}(t))}{\partial t}
    &= -\tr \left( \lim_{\epsilon\rightarrow0^+} \frac{\partial }{\partial \epsilon
    } \frac{\partial }{\partial {\cnfx}} \left( {\cnfx} + \epsilon f({\cnfx}(t), t) + \mathcal{O}(\epsilon^2) + \mathcal{O}(\epsilon^3) + \dots \right) \right) \\
    &= -\tr \left( \lim_{\epsilon\rightarrow0^+} \frac{\partial }{\partial \epsilon
    } \left( I + \frac{\partial }{\partial {\cnfx}}\epsilon f({\cnfx}(t), t) + \mathcal{O}(\epsilon^2) + \mathcal{O}(\epsilon^3) + \dots \right) \right) \\
    &= -\tr \left( \lim_{\epsilon\rightarrow0^+} \left(\frac{\partial }{\partial {\cnfx}} f({\cnfx}(t), t) + \mathcal{O}(\epsilon) + \mathcal{O}(\epsilon^2) + \dots \right) \right) \\
    &= -\tr \left( \frac{\partial }{\partial {\cnfx}} f({\cnfx}(t), t)\right)
\end{align}
\end{proof}

\subsection{Special Cases}

\paragraph{Planar CNF.} Let $f({\cnfx}) = uh(w\tran{}{\cnfx} + b)$, then $\frac{\partial f}{\partial {\cnfx}} = u \frac{\partial h}{\partial {\cnfx}}\tran{}$. Since the trace of an outer product is the inner product, we have
\begin{align}
\frac{\partial \log p({\cnfx})}{\partial t} = -\tr \left( u \frac{\partial h}{\partial {\cnfx}}\tran{}\right) = -u\tran{} \frac{\partial h}{\partial {\cnfx}}
\end{align}
This is the parameterization we use in all of our experiments.
\paragraph{Hamiltonian CNF.}
The continuous analog of NICE~\citep{dinh2014nice} is a Hamiltonian flow, which splits the data into two equal partitions and is a volume-preserving transformation, implying that $\frac{\partial \log p({\cnfx})}{\partial t}=0$.
We can verify this.
Let
\begin{align}
\begin{bmatrix}
\frac{d{\cnfx}_{1:d}}{dt} \\ \frac{d{\cnfx}_{d+1:D}}{dt}
\end{bmatrix} = \begin{bmatrix}
f({\cnfx}_{d+1:D}) \\ g({\cnfx}_{1:d})
\end{bmatrix}
\end{align}
Then because the Jacobian is all zeros on its diagonal, the trace is zero. This is a volume-preserving flow.

\subsection{Connection to Fokker-Planck and Liouville PDEs}\label{sec:connect_fp}

The Fokker-Planck equation is a well-known partial differential equation (PDE) that describes the probability density function of a stochastic differential equation as it changes with time. We relate the instantaneous change of variables to the special case of Fokker-Planck with zero diffusion, the Liouville equation.

As with the instantaneous change of variables, let $\sol(t) \in \R^D$ evolve through time following $\frac{d\sol(t)}{dt} = f(\sol(t), t)$. Then the Liouville equation describes the change in density of $\sol$--\emph{a fixed point in space}--as a PDE,
\begin{equation}\label{eq:liouville}
\frac{\partial p(\sol, t)}{\partial t} = - \sum_{i=1}^D \frac{\partial}{\partial \sol_i} \left[ f_i(\sol, t) p(\sol, t) \right]
\end{equation}
However, \eqref{eq:liouville} cannot be easily used as it requires the partial derivatives $\frac{\partial p(\sol, t)}{\partial \sol}$, which is typically approximated using finite difference. This type of PDE has its own literature on efficient and accurate simulation~\citep{stam1999stable}.

Instead of evaluating $p(\cdot, t)$ at a fixed point, if we follow the trajectory of a particle $\sol(t)$, we obtain
\begin{equation}
\begin{split}
\frac{\partial p(\sol(t), t)}{\partial t} &= \underbrace{\frac{\partial p(\sol(t), t)}{\partial \sol(t)}\frac{\partial \sol(t)}{\partial t}}_{\text{partial derivative from first argument, } \sol(t)} + \underbrace{\vphantom{\frac{\partial p(\sol(t), t)}{\partial \sol(t)}}\frac{\partial p(\sol(t), t)}{\partial t}}_{\text{partial derivative from second argument, } t} \\
&= \cancel{\sum_{i=1}^D\frac{\partial p(\sol(t), t)}{\partial \sol_i(t)}\frac{\partial \sol_i(t)}{\partial t}}
- \sum_{i=1}^{D} \frac{\partial f_i(\sol(t), t)}{\partial \sol_i} p(\sol(t), t)
- \cancel{\sum_{i=1}^{D} f_i(\sol(t), t) \frac{\partial p(\sol(t), t)}{\partial \sol_i(t)}} \\
&= - \sum_{i=1}^{D} \frac{\partial f_i(\sol(t), t)}{\partial \sol_i} p(\sol(t), t)
\end{split}
\end{equation}
We arrive at the instantaneous change of variables by taking the log,
\begin{equation}\label{eq:arriving_at_icov}
\frac{\partial \log p(\sol(t), t)}{\partial t} = \frac{1}{p(\sol(t), t)} \frac{\partial p(\sol(t), t)}{\partial t} = - \sum_{i=1}^{D} \frac{\partial f_i(\sol(t), t)}{\partial \sol_i}
\end{equation}
While still a PDE, \eqref{eq:arriving_at_icov} can be combined with $\sol(t)$ to form an ODE of size $D+1$,
\begin{equation}
\frac{d}{dt} \begin{bmatrix}
\sol(t) \\ \log p(\sol(t), t)
\end{bmatrix} = \begin{bmatrix}
f(\sol(t), t) \\ -\sum_{i=1}^{D} \frac{\partial f_i(\sol(t), t)}{\partial \sol_i}
\end{bmatrix}
\end{equation}
Compared to the Fokker-Planck and Liouville equations, the instantaneous change of variables is of more practical impact as it can be numerically solved much more easily, requiring an extra state of $D$ for following the trajectory of $\sol(t)$. Whereas an approach based on finite difference approximation of the Liouville equation would require a grid size that is exponential in $D$.

\section{A Modern Proof of the Adjoint Method}
\label{sec:modern_adj_proof}

We present an alternative proof to the adjoint method~\citep{pontryagin1962mathematical} that is short and easy to follow.

\subsection{Continuous Backpropagation}
Let $\sol(t)$ follow the differential equation $\frac{d\sol(t)}{dt} = f(\sol(t), t, \theta)$, where $\theta$ are the parameters. We will prove that if we define an adjoint state 
\begin{equation}
\adj(t) = \frac{dL}{d\sol(t)}
\end{equation} 
then it follows the differential equation
\begin{equation}\label{eq:adj_method}
\frac{d\adj(t)}{dt} = - \adj(t)\frac{\partial f(\sol(t), t, \theta)}{\partial \sol(t)} 
\end{equation}
For ease of notation, we denote vectors as row vectors, whereas the main text uses column vectors.

\noindent
The adjoint state is the gradient with respect to the hidden state at a specified time $t$. In standard neural networks, the gradient of a hidden layer $\hidden_t$ depends on the gradient from the next layer $\hidden_{t+1}$ by chain rule
\begin{equation}
\frac{d L}{d \hidden_t} = \frac{d L}{d \hidden_{t+1}} \frac{d \hidden_{t+1}}{d \hidden_{t}}.
\end{equation}
With a continuous hidden state, we can write the transformation after an $\epsilon$ change in time as
\begin{equation}
\sol(t+\epsilon) = \int_t^{t+\epsilon} f(\sol(t), t, \theta) dt + \sol(t) = T_\epsilon(\sol(t), t)
\end{equation}
and chain rule can also be applied
\begin{equation}\label{eq:chain_rule}
\frac{d L}{d \sol(t)} = \frac{d L}{d \sol(t+\epsilon)} \frac{d \sol(t+\epsilon)}{d \sol(t)} \quad\quad\text{or}\quad\quad \adj(t) = \adj(t+\epsilon)\frac{\partial T_\epsilon(\sol(t), t)}{\partial \sol(t)}
\end{equation}
The proof of \eqref{eq:adj_method} follows from the definition of derivative:
\begin{align}
\frac{d \adj(t)}{d t} &= \lim\limits_{\epsilon\rightarrow 0^+} \frac{\adj(t+\epsilon) - \adj(t)}{\epsilon} \\
&= \lim\limits_{\epsilon\rightarrow 0^+} \frac{\adj(t+\epsilon) - \adj(t+\epsilon)\frac{\partial }{\partial \sol(t)}T_\epsilon(\sol(t))}{\epsilon} \qquad\qquad\qquad\qquad\qquad\;\quad \;\textnormal{(by Eq \ref{eq:chain_rule})}\\
&= \lim\limits_{\epsilon\rightarrow 0^+} \frac{\adj(t+\epsilon) - \adj(t+\epsilon)\frac{\partial }{\partial \sol(t)}\left( \sol(t) + \epsilon f(\sol(t), t, \theta) + \bigO(\epsilon^2) \right)}{\epsilon} \;\;\;\quad \textnormal{(Taylor series around $\sol(t)$)}\\
&= \lim\limits_{\epsilon\rightarrow 0^+} \frac{\adj(t+\epsilon) - \adj(t+\epsilon)\left( I + \epsilon \frac{\partial f(\sol(t), t, \theta) }{\partial \sol(t)} + \bigO(\epsilon^2) \right)}{\epsilon} \\
&= \lim\limits_{\epsilon\rightarrow 0^+} \frac{ - \epsilon \adj(t+\epsilon)\frac{\partial f(\sol(t), t, \theta) }{\partial \sol(t)} + \bigO(\epsilon^2)}{\epsilon} \\
&= \lim\limits_{\epsilon\rightarrow 0^+} - \adj(t+\epsilon)\frac{\partial f(\sol(t), t, \theta)}{\partial \sol(t)} + \bigO(\epsilon) \\
&= - \adj(t)\frac{\partial f(\sol(t), t, \theta)}{\partial \sol(t)}\label{eq:adjoint_z}
\end{align}
We pointed out the similarity between adjoint method and backpropagation (eq. \ref{eq:chain_rule}). Similarly to backpropagation, ODE for the adjoint state needs to be solved \textit{backwards} in time. We specify the constraint on the last time point, which is simply the gradient of the loss wrt the last time point, and can obtain the gradients with respect to the hidden state at any time, including the initial value.
\begin{equation}\label{eq:gradient_iv}
\underbrace{\adj(t_N) = \frac{dL}{d \sol(t_N)}}_{\text{initial condition of adjoint diffeq.}} \quad\quad\quad\quad \underbrace{\adj(t_0) = \adj(t_N) + \int_{t_N}^{t_0} \frac{d\adj(t)}{dt} \;dt = \adj(t_N) - \int_{t_N}^{t_0} \adj(t) \frac{\partial f(\sol(t), t, \theta)}{\partial \sol(t)} \;dt}_{\text{gradient wrt. initial value}}
\end{equation}
Here we assumed that loss function $L$ depends only on the last time point $t_N$. If function $L$ depends also on intermediate time points $t_1, t_2, \dots, t_{N-1}$, etc., we can repeat the adjoint step for each of the intervals $[t_{N-1}, t_N]$, $[t_{N-2}, t_{N-1}]$ in the backward order and sum up the obtained gradients.

\subsection{Gradients wrt. $\theta$ and $t$}

We can generalize \eqref{eq:adj_method} to obtain gradients with respect to $\theta$--a constant wrt. $t$--and the initial and end times, $t_0$ and $t_N$. We view $\theta$ and $t$ as states with constant differential equations and write
\begin{equation}
\frac{\partial \theta(t)}{\partial t} = \mathbf{0} \quad\quad \frac{dt(t)}{dt} = 1
\end{equation}
We can then combine these with $z$ to form an augmented state\footnote{Note that we've overloaded $t$ to be both a part of the state and the (dummy) independent variable. The distinction is clear given context, so we keep $t$ as the independent variable for consistency with the rest of the text.} with corresponding differential equation and adjoint state,
\begin{equation}
\frac{d}{dt} \begin{bmatrix} 
\sol \\ \theta \\ t
\end{bmatrix}(t) = f_{aug}([\sol, \theta, t]) := \begin{bmatrix}
f([\sol, \theta, t]) \\ \mathbf{0} \\ 1
\end{bmatrix}, \;
\adj_{aug} := \begin{bmatrix}
\adj \\ \adj_\theta \\ \adj_t
\end{bmatrix}, \;
\adj_\theta(t) := \frac{dL}{d\theta(t)},\;
\adj_t(t) := \frac{dL}{dt(t)}
\end{equation}
Note this formulates the augmented ODE as an autonomous (time-invariant) ODE, but the derivations in the previous section still hold as this is a special case of a time-variant ODE. The Jacobian of $f$ has the form
\begin{equation}
\frac{\partial f_{aug}}{\partial [\sol, \theta, t]} = \begin{bmatrix}
\frac{\partial f }{\partial \sol } & \frac{\partial f }{\partial \theta } & \frac{\partial f }{\partial t} \\
\mathbf{0} & \mathbf{0} & \mathbf{0} \\
\mathbf{0} & \mathbf{0} & \mathbf{0}
\end{bmatrix}(t)
\end{equation}
where each $\mathbf{0}$ is a matrix of zeros with the appropriate dimensions.
We plug this into \eqref{eq:adj_method} to obtain
\begin{equation}
\frac{d\adj_{aug}(t)}{dt} = - \begin{bmatrix}
\adj(t) & \adj_\theta(t) & \adj_t(t)
\end{bmatrix} \frac{\partial f_{aug}}{\partial [\sol, \theta, t]}(t) = - \begin{bmatrix}
\adj \frac{\partial f}{\partial \sol} &
\adj \frac{\partial f}{\partial \theta} &
\adj \frac{\partial f}{\partial t}
\end{bmatrix}(t)
\end{equation}
The first element is the adjoint differential equation~\eqref{eq:adj_method}, as expected. The second element can be used to obtain the total gradient with respect to the parameters, by integrating over the full interval and setting $\adj_\theta(t_N) = \mathbf{0}$.
\begin{equation}\label{eq:adj_params}
\frac{dL}{d\theta} = \adj_\theta(t_0) = - \int_{t_N}^{t_0} \adj(t)\frac{\partial f(\sol(t), t, \theta)}{\partial \theta} \;dt
\end{equation}
Finally, we also get gradients with respect to $t_0$ and $t_N$, the start and end of the integration interval.
\begin{equation}\label{eq:adj_t}
\frac{dL}{dt_N} = \adj(t_N) f(\sol(t_N), t_N, \theta) \quad\quad \frac{dL}{dt_0} = \adj_t(t_0) = \adj_t(t_N) - \int_{t_N}^{t_0} \adj(t) \frac{\partial f(\sol(t), t, \theta)}{\partial t} \;dt
\end{equation}
Between \eqref{eq:adj_method}, \eqref{eq:gradient_iv}, \eqref{eq:adj_params}, and \eqref{eq:adj_t} we have gradients for all possible inputs to an initial value problem solver.

\section{Full Adjoint sensitivities algorithm}

This more detailed version of Algorithm~\ref{algo1} includes gradients with respect to the start and end times of integration.

\label{sec:full-alg}
\begin{algorithm}[h]
\caption{Complete reverse-mode derivative of an ODE initial value problem}
\label{algo2}
\begin{algorithmic}
\Require dynamics parameters $\theta$, start time $\tstart$, stop time $\tend$, final state $\sol(\tend)$, loss gradient $\nicefrac{\partial L}{\partial \sol(\tend)}$
\State $\frac{\partial L}{\partial \tend} = \frac{\partial L}{\partial \sol(\tend)} \tran{} f(\sol(\tend), \tend, \theta)$ \Comment{Compute gradient w.r.t.\ $\tend$}
\State $s_0 = [\sol(\tend), \frac{\partial L}{\partial \sol(\tend)}, {\vzero}_{|\theta|}, -\frac{\partial L}{\partial \tend}]$ \Comment{Define initial augmented state}
\Function{\textnormal{aug\_dynamics}}{$[\sol(t), \adj(t), \cdot, \cdot], t, \theta$}: \Comment{Define dynamics on augmented state}
\State \textbf{return} $[f(\sol(t), t, \theta), -\adj(t)\tran{} \frac{\partial f}{\partial \sol}, -\adj(t)\tran{} \frac{\partial f}{\partial \theta}, -\adj(t)\tran{} \frac{\partial f} {\partial t}]$ \Comment{Compute vector-Jacobian products}
\EndFunction %
\State $[\sol(t_0), \frac{\partial L}{\partial \sol(t_0)}, \frac{\partial L}{\partial \theta}, \frac{\partial L}{\partial t_0}] = \solvefunc(s_0, \textnormal{aug\_dynamics}, \tend, \tstart, \theta)$ \Comment{Solve reverse-time ODE}
\Ensure $\frac{\partial L}{\partial \sol(t_0)}, \frac{\partial L}{\partial \theta}, \frac{\partial L}{\partial t_0}, \frac{\partial L}{\partial \tend}$ \Comment{Return all gradients}
\end{algorithmic}
\end{algorithm}

\pagebreak

\section{Autograd Implementation}
\label{autograd code}

\begin{lstlisting}
import scipy.integrate

import autograd.numpy as np
from autograd.extend import primitive, defvjp_argnums
from autograd import make_vjp
from autograd.misc import flatten
from autograd.builtins import tuple

odeint = primitive(scipy.integrate.odeint)


def grad_odeint_all(yt, func, y0, t, func_args, **kwargs):
    # Extended from "Scalable Inference of Ordinary Differential
    # Equation Models of Biochemical Processes", Sec. 2.4.2
    # Fabian Froehlich, Carolin Loos, Jan Hasenauer, 2017
    # https://arxiv.org/pdf/1711.08079.pdf

    T, D = np.shape(yt)
    flat_args, unflatten = flatten(func_args)

    def flat_func(y, t, flat_args):
        return func(y, t, *unflatten(flat_args))

    def unpack(x):
        #      y,      vjp_y,      vjp_t,    vjp_args
        return x[0:D], x[D:2 * D], x[2 * D], x[2 * D + 1:]

    def augmented_dynamics(augmented_state, t, flat_args):
        # Original system augmented with vjp_y, vjp_t and vjp_args.
        y, vjp_y, _, _ = unpack(augmented_state)
        vjp_all, dy_dt = make_vjp(flat_func, argnum=(0, 1, 2))(y, t, flat_args)
        vjp_y, vjp_t, vjp_args = vjp_all(-vjp_y)
        return np.hstack((dy_dt, vjp_y, vjp_t, vjp_args))

    def vjp_all(g,**kwargs):

        vjp_y = g[-1, :]
        vjp_t0 = 0
        time_vjp_list = []
        vjp_args = np.zeros(np.size(flat_args))

        for i in range(T - 1, 0, -1):

            # Compute effect of moving current time.
            vjp_cur_t = np.dot(func(yt[i, :], t[i], *func_args), g[i, :])
            time_vjp_list.append(vjp_cur_t)
            vjp_t0 = vjp_t0 - vjp_cur_t

            # Run augmented system backwards to the previous observation.
            aug_y0 = np.hstack((yt[i, :], vjp_y, vjp_t0, vjp_args))
            aug_ans = odeint(augmented_dynamics, aug_y0,
                             np.array([t[i], t[i - 1]]), tuple((flat_args,)), **kwargs)
            _, vjp_y, vjp_t0, vjp_args = unpack(aug_ans[1])

            # Add gradient from current output.
            vjp_y = vjp_y + g[i - 1, :]

        time_vjp_list.append(vjp_t0)
        vjp_times = np.hstack(time_vjp_list)[::-1]

        return None, vjp_y, vjp_times, unflatten(vjp_args)
    return vjp_all


def grad_argnums_wrapper(all_vjp_builder):
    # A generic autograd helper function.  Takes a function that
    # builds vjps for all arguments, and wraps it to return only required vjps.
    def build_selected_vjps(argnums, ans, combined_args, kwargs):
        vjp_func = all_vjp_builder(ans, *combined_args, **kwargs)
        def chosen_vjps(g):
            # Return whichever vjps were asked for.
            all_vjps = vjp_func(g)
            return [all_vjps[argnum] for argnum in argnums]
        return chosen_vjps
    return build_selected_vjps

defvjp_argnums(odeint, grad_argnums_wrapper(grad_odeint_all))


\end{lstlisting}

\section{Algorithm for training the latent ODE model}
\label{lode alg}

To obtain the latent representation $\sol_{t_0}$, we traverse the sequence using RNN and obtain parameters of the distribution $q(\sol_{t_0}| \{\vx_{t_i}, t_i\}_i, \theta_{enc})$. The algorithm follows a standard VAE algorithm with an RNN variational posterior and an ODESolve model:
\begin{enumerate}
\item Run an RNN encoder through the time series and infer the parameters for a posterior over $\sol_{t_0}$:
\begin{align} q(\sol_{t_0}| \{\vx_{t_i}, t_i\}_i, \phi) = \mathcal{N}(\sol_{t_0} | \mu_{\sol_{t_0}}, \sigma_{\sol_0}), \end{align}  where
$\mu_{\sol_0}, \sigma_{\sol_0}$ come from the hidden state of RNN($\{\vx_{t_i}, t_i\}_i, \phi$)
\item Sample $\sol_{t_0} \sim q(\sol_{t_0}| \{\vx_{t_i}, t_i\}_i)$
\item Obtain $\sol_{t_1}, \sol_{t_2}, \dots, \sol_{t_M}$ by solving ODE $\solvefunc(\sol_{t_0}, f, \theta_f, t_0, \dots, t_M)$, where $f$ is the function defining the gradient $d\sol/dt$ as a function of $\sol$
\item Maximize $ \textnormal{ELBO} =   \sum_{i=1}^M \log p(\vx_{t_i} | \sol_{t_i}, \theta_{\obs}) + \log p(\sol_{t_0}) - \log  q(\sol_{t_0}| \{\vx_{t_i}, t_i\}_i, \phi)$, \\
where $p(\sol_{t_0}) = \mathcal{N}(0,1)$
\end{enumerate}

\section{Extra Figures}
\label{seq:extra}
\begin{figure}[h!]
	\centering
	\begin{subfigure}[b]{0.29\linewidth}
		\centering
		\includegraphics[trim={0.7cm 0.7cm 0.7cm 0.9cm},clip,width=1.\textwidth]{plots/ode_30sp_predictions_test_11}\\
		\includegraphics[trim={0.7cm 0.7cm 0.7cm 0.9cm},clip,width=1.\textwidth]{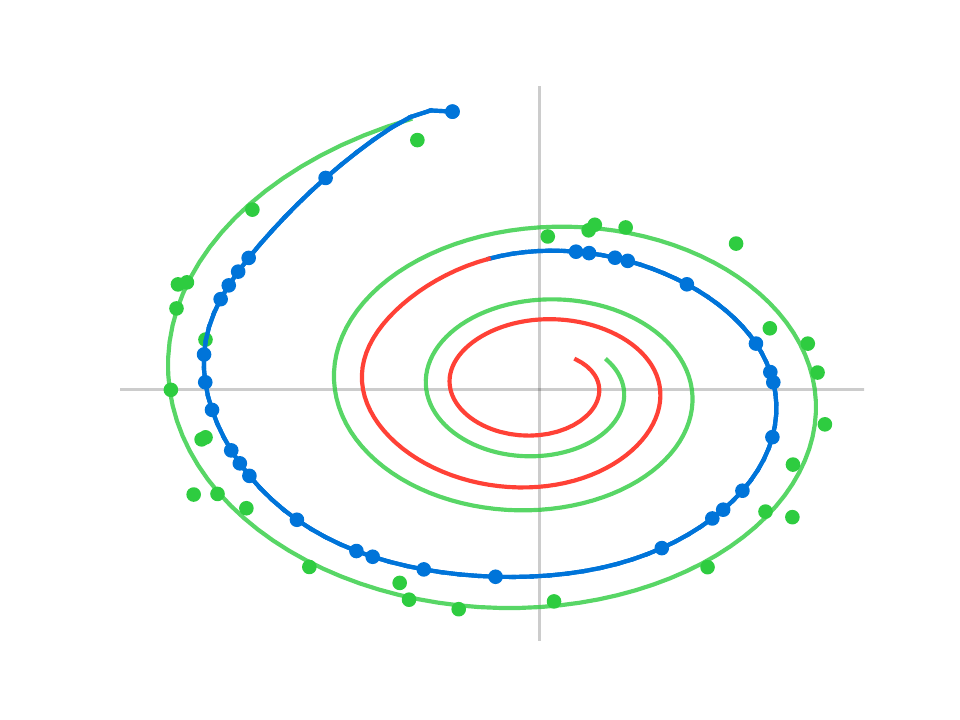}
		\caption{30 time points}
	\end{subfigure}
	\begin{subfigure}[b]{0.29\linewidth}
		\centering
		\includegraphics[trim={0.7cm 0.7cm 0.7cm 0.9cm},clip,width=1.\textwidth]{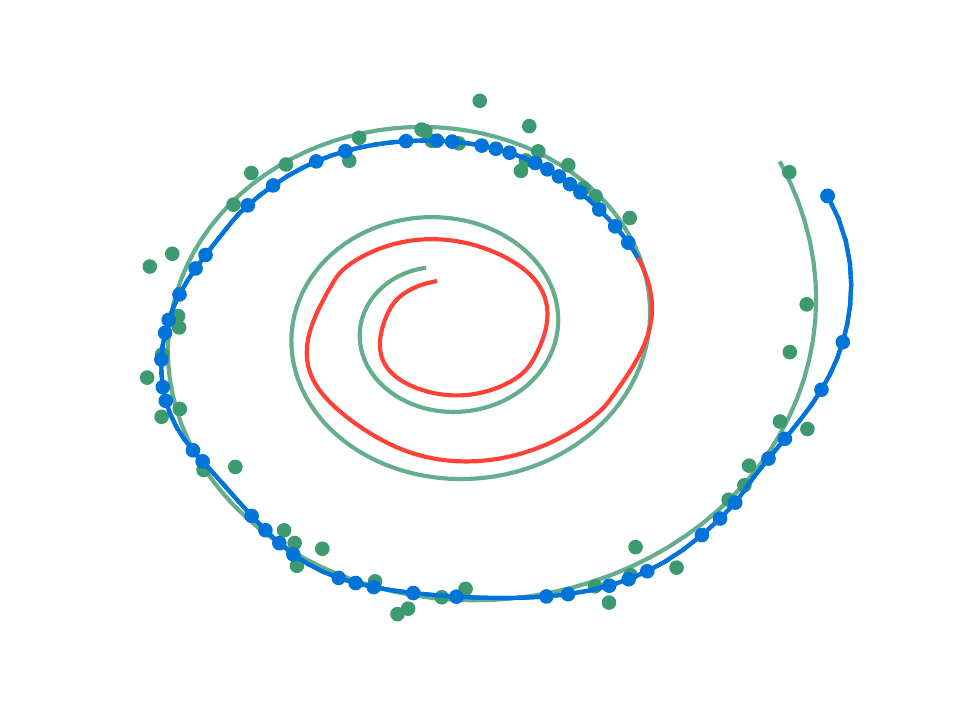}\\
		\includegraphics[trim={0.7cm 0.7cm 0.7cm 0.9cm},clip,width=1.\textwidth]{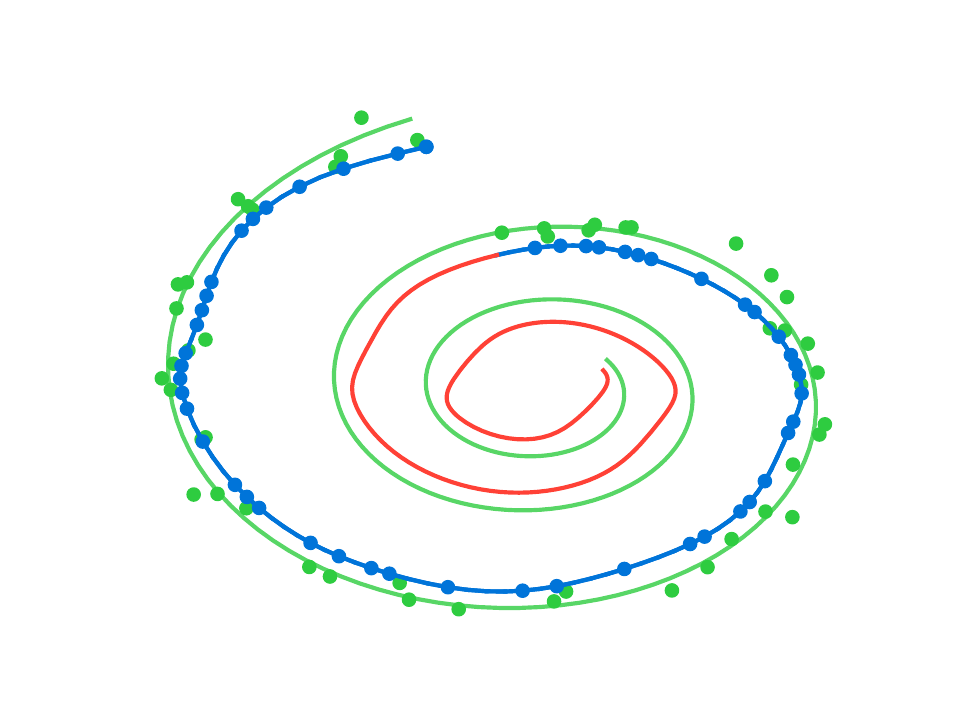}
		\caption{50 time points}
	\end{subfigure}
	\begin{subfigure}[b]{0.29\linewidth}
		\centering
		\includegraphics[trim={0.7cm 0.7cm 0.7cm 0.9cm},clip,width=1.\textwidth]{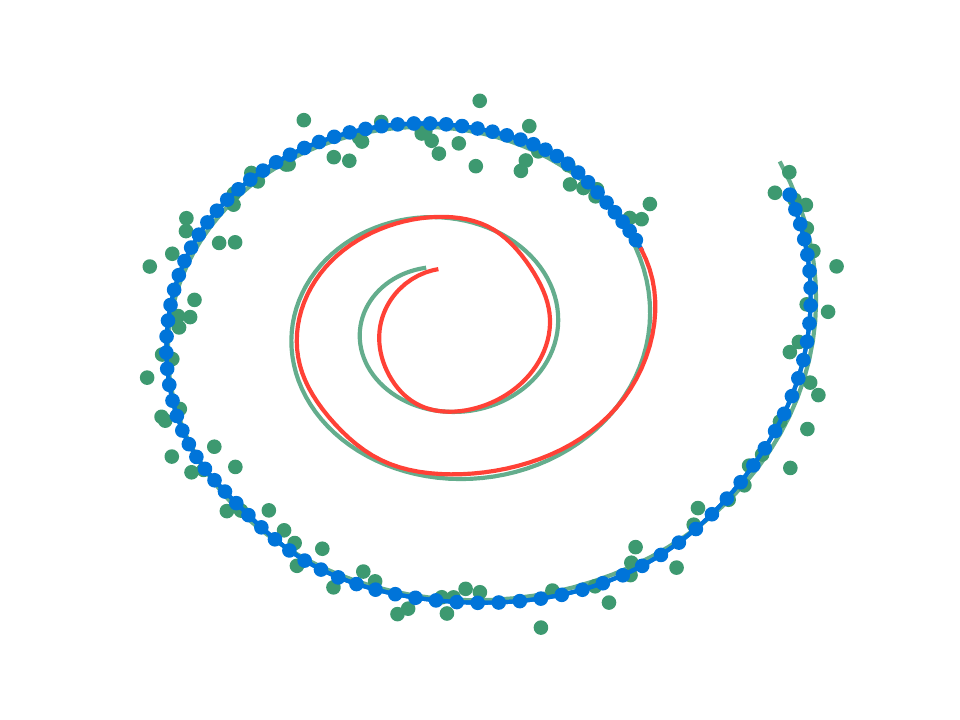}
		\includegraphics[trim={0.7cm 0.7cm 0.7cm 0.9cm},clip,width=1.\textwidth]{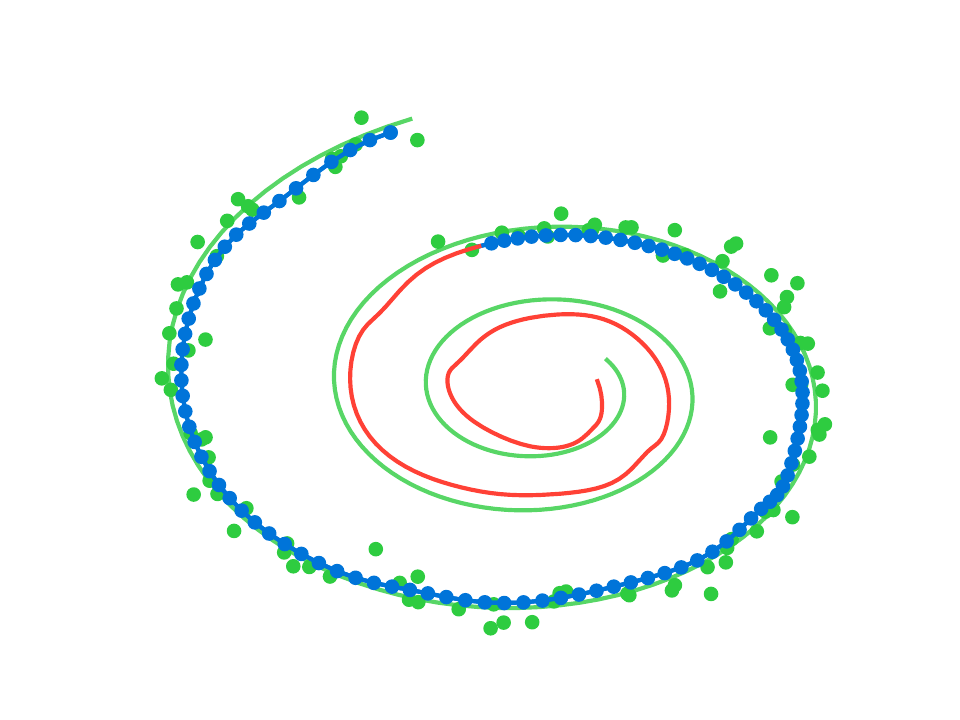}
		\caption{100 time points}
	\end{subfigure}
	\begin{subfigure}[b]{0.11\linewidth}
		\centering
		\includegraphics[trim={1.15cm 1.7cm 1.15cm 1.7cm},clip,width=1.2\textwidth]{plots/pred_legend}\vspace{4cm}
	\end{subfigure}
	\caption{Spiral reconstructions using a latent ODE with a variable number of noisy observations.}
\end{figure}

\end{appendices}

\end{document}